\documentclass[11pt,notitlepage]{article}
\usepackage{color}
\usepackage{fullpage}
\usepackage{authblk}

\usepackage[utf8]{inputenc} % allow utf-8 input
\usepackage{url}            % simple URL typesetting
\usepackage{booktabs}       % professional-quality tables
\usepackage{amsfonts}       % blackboard math symbols
\usepackage{nicefrac}       % compact symbols for 1/2, etc.
\usepackage{microtype}      % microtypography

\usepackage{amsmath,amssymb, amsthm}
\usepackage{bbm}
\usepackage{graphicx}
\usepackage{comment}
\usepackage{ifthen, calc}
\usepackage{xcolor}
\usepackage{pstricks}
\usepackage{epsfig}
\usepackage{pst-grad} % For gradients
\usepackage{pst-plot} % For axes
\usepackage{hyperref}
\usepackage{mathrsfs}
\usepackage{mathtools}
\usepackage{multicol}
\usepackage{relsize}
\usepackage{tikz}

\usepackage{color}
\newcommand\Ls {\mathscr{L}}

\newcommand\Sc {\mathcal{S}}
\newcommand\Nc {\mathcal{N}}
\newcommand\bbH{\mathbb{H}}
\newcommand{\htheta}{\hat{\theta}}
\newcommand{\floor}[1]{\left\lfloor #1 \right\rfloor}

\newcommand\E {\mathbb{E}}
\newcommand\Expect{\E}
\newcommand\Pbb {\mathbb{P}}
\newcommand\R {\mathbb{R}}
\renewcommand\Re{\R}
\newcommand\bs {\boldsymbol}

\newcommand\sam {\ensuremath{\operatorname{\mathrm{sam}}}}
\newcommand\eps {\varepsilon}

\newcommand\ind {\mathbbm{1}}

\newcommand*{\eqdef}{\stackrel{\textup{def}}{=}}

\newcommand\rank {\ensuremath{\operatorname{\mathrm{rank}}}}

\newcommand\tr {\ensuremath{\operatorname{\mathrm{tr}}}}
\newcommand\diag {\ensuremath{\operatorname{\mathrm{diag}}}}

\newtheorem{theorem}{Theorem}
\newtheorem{definition}[theorem]{Definition}
\newtheorem{lemma}[theorem]{Lemma}

\newtheorem{cor}[theorem]{Corollary}

\begin{document}

\title{Benign Overfitting in Linear Regression}
\date{}
% \title{Overparameterization and Benign Overfitting in Linear Regression}

% Use letters for affiliations, numbers to show equal authorship (if applicable) and to indicate the corresponding author
\author[a,b]{Peter L.~Bartlett}
\author[c]{Philip M.~Long}
\author[d]{G\'abor Lugosi}
\author[a]{Alexander Tsigler}

\affil[a]{Department of Statistics, UC Berkeley, 367 Evans Hall,
Berkeley CA 94720-3860}
\affil[b]{Computer Science Division, UC Berkeley, 387 Soda Hall,
Berkeley CA 94720-1776}
\affil[c]{Google}
\affil[d]{Economics and Business,
Pompeu Fabra University;
ICREA, Pg.~Llu\'is Companys 23, 08010 Barcelona, Spain;
Barcelona Graduate School of Economics}

% Keywords are not mandatory, but authors are strongly encouraged to provide them. If provided, please include two to five keywords, separated by the pipe symbol, e.g:
% \keywords{statistical learning theory $|$ overfitting $|$ linear regression 
%           $|$ interpolation} 

\maketitle

\begin{abstract}
The phenomenon of benign overfitting is one of the key
mysteries uncovered by deep learning methodology: deep
neural networks seem to predict well, even with a perfect
fit to noisy training data. Motivated by this phenomenon,
we consider when a perfect fit to training data in
linear regression is compatible with accurate prediction.
We give a characterization of linear regression
problems for which the minimum norm interpolating prediction rule
has near-optimal prediction accuracy. The characterization is in
terms of two notions of the effective rank of the data covariance.
It shows that overparameterization is essential for
benign overfitting in this setting: the number of directions
in parameter space 
that are unimportant for prediction must
significantly exceed the sample size.  By studying examples of data
covariance properties that this characterization shows are required
for benign overfitting, we find an important role for
finite-dimensional data: the accuracy of the minimum norm
interpolating prediction rule approaches the best possible accuracy
for a much narrower range of properties of the data distribution
when the data lies in an infinite dimensional space versus when the
data lies in a finite dimensional space whose dimension grows faster
than the sample size.
\end{abstract}

% If your first paragraph (i.e. with the \dropcap) contains a list environment (quote, quotation, theorem, definition, enumerate, itemize...), the line after the list may have some extra indentation. If this is the case, add \parshape=0 to the end of the list environment.
%\dropcap{D}eep
\section{Introduction}
Deep
learning methodology has revealed a surprising statistical
phenomenon: overfitting can perform well.
The classical perspective in statistical learning theory is that
there should be a tradeoff between the fit to the training data
and the complexity of the prediction rule. Whether complexity
is measured in terms of the number of parameters, the number of
non-zero parameters in a high-dimensional setting, the number of
neighbors averaged in a nearest-neighbor estimator, the scale of an
estimate in a reproducing kernel Hilbert space, or the bandwidth of
a kernel smoother, this tradeoff has been ubiquitous
in statistical learning theory.  Deep learning
seems to operate outside the regime where results of this kind are
informative, since deep neural networks can perform well even with
a perfect fit to the training data.

As one example of this phenomenon, consider the experiment
illustrated in Figure~1(c) in~\cite{zbhrv-udlrrg-17}: standard deep
network architectures and stochastic gradient algorithms, run until
they perfectly fit a standard image classification training set,
give respectable prediction performance, {\em even when significant
levels of label noise are introduced}.  The deep networks in the
experiments reported in~\cite{zbhrv-udlrrg-17} achieved essentially
zero cross-entropy loss on the training data.  In statistics
and machine learning textbooks, an estimate that fits every
training example perfectly is often presented as an illustration
of overfitting (``... interpolating fits... [are] unlikely to
predict future data well at all.''~\cite[p37]{htf-esl-01}). Thus,
to arrive at a scientific understanding of the success of deep
learning methods, it is a central challenge to understand the
performance of prediction rules that fit the training data perfectly.

In this paper, we consider perhaps the simplest setting where we
might hope to witness this phenomenon: linear regression.
That is, we consider quadratic loss and linear prediction rules, and
we assume that the dimension of the parameter space is large enough
that a perfect fit is guaranteed. We consider data in an infinite
dimensional space (a separable Hilbert space), but our results
apply to a finite-dimensional subspace as a special case. There
is an ideal value of the parameters, $\theta^*$, corresponding to
the linear prediction rule that minimizes the expected quadratic loss.
We ask when it is possible to fit the data exactly
and still compete with the prediction accuracy of $\theta^*$.
Since we require more parameters than the sample size in order to
fit exactly, the solution might be underdetermined, so there might
be many interpolating solutions. We consider the most natural: 
choose the parameter vector $\hat\theta$ with the smallest norm
among all vectors that give perfect predictions on the training
sample. (This corresponds to using the pseudoinverse to solve
the normal equations; see Section~\ref{section:notation}.)
We ask when it is possible to overfit in this
way---and embed all of the noise of the labels into the parameter
estimate $\hat\theta$---without harming prediction accuracy.

Our main result is a finite sample characterization of when
overfitting is benign in this setting. The linear regression
problem depends on the optimal parameters $\theta^*$ and the
covariance $\Sigma$ of the covariates $x$. The properties of $\Sigma$
turn out to be crucial, since the magnitude of the variance in
different directions determines both how the label noise gets
distributed across the parameter space and how errors in
parameter estimation in different directions in parameter space
affect prediction accuracy. There is a classical decomposition of
the excess prediction error into two terms. The
first is rather standard: provided that the scale of the problem
(that is, the sum of the eigenvalues of $\Sigma$) is small compared
to the sample size $n$, the contribution to $\hat\theta$ that we can
view as coming from $\theta^*$ is not too distorted. The second
term is more interesting, since it reflects the impact of the noise
in the labels on prediction accuracy.
We show that this part is small if and only if the
effective rank of $\Sigma$ in the subspace corresponding to low
variance directions is large compared to $n$.  This necessary
and sufficient condition of a large effective rank can be viewed
as a property of significant overparameterization: fitting the
training data exactly but with near-optimal prediction accuracy
occurs if and only if there are many low variance (and hence
unimportant) directions in parameter space where the label
noise can be hidden.

The details are more complicated. The characterization
depends in a specific way on {\em two} notions of effective rank,
$r$ and $R$; the smaller one, $r$, determines a split of $\Sigma$
into large and small eigenvalues, and the excess prediction error
depends on the effective rank, as measured by the larger notion $R$, of
the subspace corresponding to the smallest eigenvalues. For the excess
prediction error to be small, the smallest eigenvalues of $\Sigma$
must decay slowly.

Studying the patterns of eigenvalues that allow benign overfitting
reveals an interesting role for large but finite dimensions: in an
infinite-dimensional setting, benign overfitting occurs only for a
narrow range of decay rates of the eigenvalues. On the other
hand, it occurs with any suitably slowly decaying eigenvalue sequence
in a finite dimensional space whose dimension grows
faster than the sample size. Thus, for linear regression,
data that lies in a large but finite dimensional space exhibits the
benign overfitting phenomenon with a much wider range of covariance
properties than data that lies in an infinite dimensional space.

The phenomenon of interpolating prediction rules has been an
object of study by several authors over the last two years, since
it emerged as an intriguing mystery at the Simons Institute program
on Foundations of Machine Learning in Spring 2017.  Belkin, Ma and
Mandal~\cite{pmlr-v80-belkin18a} 
described an experimental study
demonstrating that this phenomenon of accurate prediction for
functions that interpolate noisy data also occurs for prediction
rules chosen from reproducing kernel Hilbert spaces, and explained
the mismatch between this phenomenon and classical generalization
bounds.  Belkin, Hsu and Mitra~\cite{bhm-opfrbcrri-18} gave
an example of an interpolating decision rule---simplicial
interpolation---with an asymptotic consistency property as the input dimension
gets large. That work, and subsequent work of Belkin, Rakhlin, and
Tsybakov~\cite{brt-ddicso-18}, studied
kernel smoothing methods based on singular kernels that
both interpolate and, with suitable bandwidth choice, give optimal
rates for nonparametric estimation (building on earlier consistency
results~\cite{dgk-hkre-98} for these unusual kernels).  Liang and
Rakhlin~\cite{lr-jikrrcg-18} considered minimum norm interpolating
kernel regression with kernels defined as nonlinear functions of the
Euclidean inner product and showed that, with certain properties
of the training sample (expressed in terms of the empirical kernel
matrix), these methods can have good prediction accuracy.
Belkin, Hsu, Ma and Mandal~\cite{bhmm-rmmlbvt-18} studied
experimentally the excess risk as a function of the dimension of a
sequence of parameter spaces for linear and non-linear classes.

Subsequent
to our work,~\cite{mvs-hindr-19} considered the properties of the
interpolating linear prediction rule with minimal expected squared
error. After this work was presented at the NAS Colloquium on
the Science of Deep Learning~\cite{bllt-apincslt-19}, we became
aware of the concurrent work of Belkin, Hsu and
Xu~\cite{bhx-tmddwf-19} and of Hastie, Montanari, Rosset and
Tibshirani~\cite{hmrt-shdrlsi-19}.
Belkin {\it et al}~\cite{bhx-tmddwf-19} calculated the excess risk for
certain linear models (a regression problem with identity covariance,
sparse optimal parameters, both with and without noise, and a
problem with random Fourier features with no noise), and
Hastie {\it et al} considered linear regression
in an asymptotic regime, where sample size $n$ and input dimension
$p$ go to infinity together with asymptotic ratio $p/n\to\gamma$.
They assumed that, as $p$ gets large, the empirical spectral distribution
of $\Sigma$ (the discrete measure on its set of eigenvalues)
converges to a fixed measure, and they applied random matrix theory
to explore the range of behaviors of the asymptotics of the
excess prediction error as $\gamma$, the noise variance, and the
eigenvalue distribution vary. They also studied the asymptotics of a
model involving random nonlinear features.  In contrast,
we give upper and lower bounds on the excess prediction error
for arbitrary finite sample size, for arbitrary covariance matrices,
and for data of arbitrary dimension.

The next section introduces notation and definitions used
throughout the paper, including definitions of the problem of
linear regression and of various notions
of effective rank of the covariance operator.
Section~\ref{section:results} gives the characterization of benign
overfitting, illustrates why the effective rank condition
corresponds to significant overparameterization, and presents several
examples of patterns of eigenvalues that allow benign overfitting,
suggesting that slowly decaying covariance eigenvalues
in input spaces of growing but finite dimension are the generic
example of benign overfitting.
Section~\ref{section:deepnets} discusses the connections between these
results and the benign overfitting phenomenon in deep neural networks.
Section~\ref{section:proof} outlines the proofs of the results.

\section{Definitions and Notation}\label{section:notation}

We consider linear regression problems, where a linear function of
covariates $x$ from a (potentially infinite dimensional) Hilbert
space $\bbH$ is used to predict a real-valued response variable $y$.
We use vector notation, so that $x^\top\theta$ denotes the inner
product between $x$ and $\theta$
and $xz^\top$ denotes the tensor product of $x,z\in \bbH$.

\begin{definition}[Linear regression]\label{def:linreg}
  A linear regression problem in a separable Hilbert space $\bbH$
  is defined by a random covariate vector $x\in\bbH$ and outcome
  $y\in\Re$. We define
    \begin{enumerate}
      \item the {\em covariance operator} $\Sigma=\Expect[x x^\top]$,
      and
      \item the {\em optimal parameter vector}
      $\theta^*\in\bbH$, satisfying
      $\Expect(y-x^\top\theta^*)^2=\min_\theta
      \Expect(y-x^\top\theta)^2$.
    \end{enumerate}
   We assume
   \begin{enumerate}
   	\item\label{assumption:meanzero}
    $x$ and $y$ are mean zero;
    \item\label{assumption:subgaussiandata}
    $x = V\Lambda^{1/2} z$, where $\Sigma=V\Lambda V^\top$
    is the spectral decomposition of $\Sigma$ and $z$ has components
    that are independent $\sigma_x^2$-subgaussian with $\sigma_x$ a
    positive constant; that is, for all $\lambda \in \bbH$, 
      \[
    	\E[\exp(\lambda^\top z)] \leq \exp(\sigma_x^2\|\lambda\|^2/2),
   	  \]
   	where $\|\cdot\|$ is the norm in the Hilbert space $\bbH$;
   	\item%\label{assumption:noisevariance}
    the {\em conditional noise variance} is bounded below by some constant $\sigma^2$,
   	\[
	   	\Expect\left[(y-x^\top\theta^*)^2\middle|x\right] \geq \sigma^2;
	\]
   	\item%\label{assumption:subgaussiannoise}
    $y - x^\top\theta^\ast$ is $\sigma_y^2$-subgaussian, conditionally on $x$, that is for all $\lambda \in \Re$
   	\[
   	\E[\exp(\lambda(y - x^\top\theta^\ast))|x] \leq \exp(\sigma_y^2\lambda^2/2)
   	\]
   	(note that this implies $\E[y|x] = x^\top\theta^*$);
    \item \label{assumption:full_rank}
  almost surely, the projection of the
  data $X$ on the space orthogonal to any eigenvector of $\Sigma$
  spans a space of dimension $n$.
   	\end{enumerate}
  Given a training sample $(x_1,y_1),\ldots,(x_n,y_n)$
  of $n$ i.i.d.~pairs with the same distribution as $(x,y)$,
  an {\em estimator} returns a parameter estimate $\theta\in\bbH$.
  The excess risk of the estimator is defined as
    \[
      R(\theta):=
        \Expect_{x,y}\left[
        \left(y-x^\top\theta\right)^2
        - \left(y-x^\top\theta^*\right)^2\right],
    \]
  where $\Expect_{x,y}$ denotes the conditional expectation given
  all random quantities other than $x,y$ (in this case, given the
  estimate $\theta$).
  Define the vectors $\bs{y}\in\Re^n$ with entries $y_i$ and
  $\bs{\eps}\in\Re^n$ with entries $\eps_i=y_i-
  x_i^\top\theta^*$.  We use infinite matrix notation: $X$
  denotes the linear map from $\bbH$ to $\Re^n$ corresponding to
  $x_1,\ldots,x_n\in\bbH$, so that $X\theta\in\Re^n$ has $i$th component
  $ x_i^\top\theta$. We use similar notation for the
  linear map $X^\top$ from $\Re^n$ to $\bbH$.
\end{definition}
Notice that Assumptions~\ref{assumption:meanzero}
to~\ref{assumption:full_rank} are satisfied when $x$ and $y$ are
jointly gaussian with zero mean and $\rank(\Sigma)>n$.

We shall be concerned with situations where an estimator $\theta$
can fit the data perfectly, that is, $X\theta=\bs{y}$.
Typically this implies that there are many such vectors. We consider the
interpolating estimator with minimal norm in $\bbH$.
We use $\|\cdot\|$ to denote both the
Euclidean norm of a vector in $\Re^n$ and the Hilbert space norm.

\begin{definition}[Minimum norm estimator]\label{def:lne}
  Given data $X\in\bbH^n$ and $\bs{y}\in\Re^n$,
  the {\em minimum norm estimator} $\hat\theta$ solves the optimization problem
    \begin{align*}
      \min_{\theta\in\bbH} \qquad & \left\|\theta\right\|^2 \\
      \text{such that }\qquad & \left\|X\theta-\bs{y}\right\|^2
        = \min_\beta \left\|X\beta-\bs{y}\right\|^2.
    \end{align*}
\end{definition}
By the projection theorem, parameter vectors that solve the least squares
problem $\min_\beta \left\|X\beta-\bs{y}\right\|^2$ solve the normal
equations, so we can equivalently write $\hat\theta$ as the minimum norm
solution to the normal equations,
  \begin{align*}
    \hat\theta
      & =\arg\min_\theta\left\{\left\|\theta\right\|^2:
        X^\top X\theta = X^\top \bs{y} \right\} \\
      & = \left(X^\top X\right)^\dagger X^\top \bs{y} \\
      & = X^\top\left(XX^\top\right)^\dagger\bs{y},
  \end{align*}
where $\left(X^\top X\right)^\dagger$ denotes the pseudoinverse of the
bounded linear operator $X^\top X$ (for infinite dimensional $\bbH$,
the existence of the pseudoinverse is guaranteed because $X^\top X$ is
bounded and has a closed range; see~\cite{dw-np-63}). When $\bbH$ has dimension $p$ with $p<n$ and $X$ has rank $p$, there is a unique
solution to the normal equations. 
On the contrary,
Assumption~\ref{assumption:full_rank} in Definition~\ref{def:linreg}
implies that we can find many solutions $\theta\in\bbH$ to the normal
equations that achieve $X\theta=y$.  
The minimum norm solution is given by
  \begin{align}
    \hat\theta & = X^\top\left(XX^\top\right)^{-1}\bs{y}.
            \label{equation:hattheta}
  \end{align}

Our main result gives tight bounds on the excess risk of this minimum
norm estimator in terms of certain notions of effective rank of the
covariance that are defined in terms of its eigenvalues.

We use $\mu_1(\Sigma)\ge\mu_2(\Sigma)\ge \cdots$ to denote the
eigenvalues of $\Sigma$ in descending order, and we denote the
operator norm of $\Sigma$ by $\|\Sigma\|$. We use $I$ to denote
the identity operator on $\bbH$ and $I_n$ to denote the $n\times n$
identity matrix.

\begin{definition}[Effective Ranks]\label{def:ranks}
For the covariance operator $\Sigma$,
define $\lambda_i=\mu_i(\Sigma)$ for $i=1,2,\ldots$.
If $\sum_{i=1}^\infty\lambda_i<\infty$ and $\lambda_{k+1}>0$
for $k\ge 0$, define
  \begin{align*}
    r_k(\Sigma) & = \frac{\sum_{i>k}\lambda_i}{\lambda_{k+1}}, &
    R_k(\Sigma) & = \frac{\left(\sum_{i>k}\lambda_i\right)^2}
        {\sum_{i>k}\lambda_i^2}.
  \end{align*}
\end{definition}

\section{Main Results}\label{section:results}

The following theorem establishes nearly matching upper and lower
bounds for the risk of the minimum-norm interpolating estimator.

\begin{theorem}\label{th::main}
 For any $\sigma_x$ there are  $b,c, c_1>1$ for which the following
 holds. Consider a linear regression problem from
 Definition~\ref{def:linreg}. Define
   \[
     k^* = \min\left\{k\ge 0: r_k(\Sigma)\ge bn\right\},
   \]
 where the minimum of the empty set is defined as $\infty$.
 Suppose $\delta<1$ with $\log(1/\delta)<n/c$.
 If $k^* \geq n/c_1$, then
 $\Expect R(\hat\theta)\ge \sigma^2/c$.
 Otherwise,
   \begin{align*}
     R(\hat\theta)
       & \le c\left(
         \|\theta^\ast\|^2\left\|\Sigma\right\|
         \max\left\{\sqrt{\frac{r_0(\Sigma)}{n}},
         \frac{r_0(\Sigma)}{n}, \sqrt{\frac{\log(1/\delta)}{n}}
         \right\}\right) 
%        \\*
%        & \qquad \qquad
         {} + c\log(1/\delta)\sigma^2_y\left(\frac{k^*}{n}
         + \frac{n}{R_{k^*}(\Sigma)}\right)
   \end{align*}
 with probability at least $1-\delta$, and
   \[
     \Expect R(\hat\theta)\ge \frac{\sigma^2}{c}
       \left(\frac{k^*}{n} + \frac{n}{R_{k^*}(\Sigma)}\right).
   \]

Moreover, there are universal constants $a_1,a_2,n_0$
such that
for all $n \geq n_0$, for all $\Sigma$, 
for all $t \geq 0$, there is
a $\theta^*$ with 
$\| \theta^* \| = t$ such that for $x\sim{\cal N}(0,\Sigma)$ and
$y|x\sim{\cal N}(x^\top\theta^*,\| \theta^* \|^2 \| \Sigma \|)$, with probability at least $1/4$,
  \begin{align*}
    R(\hat\theta)
      & \ge \frac{1}{a_1}
        \|\theta^\ast\|^2\left\|\Sigma\right\|
        \ind\left[\frac{r_0(\Sigma)}{n
        \log\left(1+r_0(\Sigma)\right)}\ge a_2\right].
    \end{align*}
\end{theorem}

\subsection{Effective Ranks and Overparameterization}

In order to understand the implications of Theorem \ref{th::main}, 
we now study relationships between the two notions of effective rank,
$r_k$ and $R_k$, and establish sufficient and necessary conditions for
the sequence $\{\lambda_i\}$ of eigenvalues to lead to small excess risk.

The following lemma shows that the two notions of effective rank are closely
related. See Appendix~\ref{appendix:rankfacts} for its proof, and for
other properties of $r_k$ and $R_k$.  (All appendices may be found in 
the supporting material.)

\begin{lemma}\label{lemma:rkRkordering}
  $r_k(\Sigma)\ge 1$, $r^2_k(\Sigma)=r_k(\Sigma^2)R_k(\Sigma)$, and
    \[
      r_k(\Sigma^2) \le r_k(\Sigma) \le R_k(\Sigma) \le r_k^2(\Sigma).
    \]
\end{lemma}

Notice that $r_0(I_p)=R_0(I_p)=p$. More generally, if all the non-zero
eigenvalues of $\Sigma$ are identical, then
$r_0(\Sigma)=R_0(\Sigma)=\rank(\Sigma)$.
For $\Sigma$ with finite rank, we can express both $r_0(\Sigma)$
and $R_0(\Sigma)$ as a product of the rank and a notion of symmetry.
In particular, for $\rank(\Sigma)=p$ we can write
  \begin{align*}
    r_0(\Sigma) &= \rank(\Sigma) s(\Sigma), &
    R_0(\Sigma) &= \rank(\Sigma) S(\Sigma), \\
\text{with} \;\;
    s(\Sigma) &= \frac{\frac{1}{p} \sum_{i=1}^p\lambda_i}{\lambda_1}, &
    S(\Sigma) &= \frac{\left(\frac{1}{p}\sum_{i=1}^p\lambda_i\right)^2}
        {\frac{1}{p}\sum_{i=1}^p\lambda_i^2}.
  \end{align*}
Both notions of symmetry $s$ and $S$ lie between $1/p$ (when
$\lambda_2\to 0$) and $1$ (when the $\lambda_i$ are all equal).

Theorem~\ref{th::main} shows that, for the minimum norm estimator
to have near-optimal prediction accuracy, $r_0(\Sigma)$ should be
small compared to the sample size $n$ (from the first term) and
$r_{k^*}(\Sigma)$ and $R_{k^*}(\Sigma)$ should be large compared to
$n$. Together, these conditions imply that overparameterization is
essential for benign overfitting in this setting: the number of
non-zero eigenvalues should be large compared to $n$, they should
have a small sum
compared to $n$, and there should be many eigenvalues no larger
than $\lambda_{k^*}$.  If the number of these small eigenvalues
is not much larger than $n$, then they should be roughly equal,
but they can be more assymmetric if there are many more of them.

The following theorem shows that the kind of overparameterization that
is essential for benign overfitting requires $\Sigma$ to have a heavy
tail. (The proof---and some other examples illustrating the boundary of
benign overfitting---are in Appendix~\ref{section:eigenproof}.) In
particular, if we fix $\Sigma$ in an infinite-dimensional Hilbert
space and ask when does the excess risk of the minimum norm estimator
approach zero as $n\to\infty$, it imposes tight restrictions on the
eigenvalues of $\Sigma$.  But there are many other possibilities for
these asymptotics if $\Sigma$ can change with $n$.
Since rescaling $X$ affects the accuracy of the least-norm
interpolant in an obvious
way, we may assume without loss of generality that
$\| \Sigma \| = 1$.  If we restrict our attention to
this case, then, informally, 
Theorem~\ref{th::main} implies that,
when the covariance operator for data with $n$ examples
is $\Sigma_n$,
the least-norm interpolant converges if $\frac{r_0(\Sigma_n)}{n} \rightarrow 0$,
$\frac{k_n^*}{n} \rightarrow 0$, 
and $\frac{n}{R_{k_n^*}(\Sigma_n)} \rightarrow 0$,
and only if 
$\frac{r_0(\Sigma_n)}{n \log (1 + r_0(\Sigma_n))} \rightarrow 0$,
$\frac{k_n^*}{n} \rightarrow 0$, 
and $\frac{n}{R_{k_n^*}(\Sigma_n)} \rightarrow 0$,
where $k^*_n= \min\left\{k\ge 0: r_k(\Sigma_n)\ge bn\right\}$ for the
universal constant $b$ in Theorem~\ref{th::main}.
For this reason, we say that a sequence of covariance operators $\Sigma_n$ is
{\em benign} if
  \[
\lim_{n \rightarrow \infty} \frac{r_0(\Sigma_n)}{n}
 = \lim_{n \rightarrow \infty} \frac{k_n^*}{n}
 = \lim_{n \rightarrow \infty} \frac{n}{R_{k_n^*}(\Sigma_n)}
 = 0.
  \]
  \begin{theorem}\label{theorem:benign_eigenvalues}
    \mbox{}
    \begin{enumerate}
      \item\label{eigenexample:infdim}
        If $\mu_k(\Sigma) = k^{-\alpha} \ln^{-\beta} (k+1)$, then
        $\Sigma$ is benign iff $\alpha=1$ and $\beta > 1$.
      \item\label{eigenexample:expplusconst}
        If
          \[
            \mu_k(\Sigma_n) = \begin{cases}
              \gamma_k + \epsilon_n & \text{if $k\le p_n$,} \\
              0   & \text{otherwise,}
              \end{cases}
          \]
        and $\gamma_k=\Theta(\exp(-k/\tau))$,
        then $\Sigma_n$ is benign iff $p_n=\omega(n)$ and
        $ne^{-o(n)}=\epsilon_np_n=o(n)$.
        Furthermore, for $p_n=\Omega(n)$ and
        $\epsilon_np_n=ne^{-o(n)}$,
          \[
            R(\hat\theta)
              \!=\! O\left(\frac{\epsilon_np_n+1}{n}
              + \frac{\ln(n/(\epsilon_np_n))}{n}
              + \max\left\{\frac{1}{n},\frac{n}{p_n}\right\}\right).
          \]
    \end{enumerate}
  \end{theorem}

\iffalse
Many popular learning algorithms may be viewed as the combination
of a feature transformation with a linear method applied to the
transformed features.  In practice, hyperparameter choices 
such as kernel parameters and deep network architectures can be made
with the knowledge of $n$, and of course these affect the spectrum of
the covariance of the transformed features.

It is informative to compare the situations described by
Parts~\ref{eigenexample:expplusconst} and \ref{eigenexample:infdim}
of Theorem~\ref{theorem:benign_eigenvalues}. 
Part~\ref{eigenexample:expplusconst} shows that
if the data has finite dimension and a small amount
of isotropic noise is added to the covariates, even
if the eigenvalues of the original covariance operator (before the
addition of isotropic noise) decay very rapidly, benign overfitting
occurs iff both the dimension is large compared to the sample size,
and the isotropic component of the covariance is sufficiently
small---but not exponentially small---compared to the sample size.
\fi

Compare the situations described by
Parts~\ref{eigenexample:infdim} and~\ref{eigenexample:expplusconst}
of Theorem~\ref{theorem:benign_eigenvalues}.
Part~\ref{eigenexample:infdim} shows that
for infinite-dimensional data with a fixed covariance, benign
overfitting occurs iff the eigenvalues of the covariance operator
decay just slowly enough for their sum to remain finite.
Part~\ref{eigenexample:expplusconst} shows that the situation is
very different if the data has finite dimension and a small amount
of isotropic noise is added to the covariates.  In that case, even
if the eigenvalues of the original covariance operator (before the
addition of isotropic noise) decay very rapidly, benign overfitting
occurs iff both the dimension is large compared to the sample size,
and the isotropic component of the covariance is sufficiently
small---but not exponentially small---compared to the sample size.

These examples illustrate the tension between the slow decay
of eigenvalues that is needed for $k/n+n/R_k$ to be
small, and the summability of eigenvalues that is needed for
$r_0(\Sigma)/n$ to be small. There are two ways to resolve this tension. 
First, in the infinite dimensional setting, slow decay of the
eigenvalues suffices---decay just fast enough to ensure
summability---as shown by Part~\ref{eigenexample:infdim} of
Theorem~\ref{theorem:benign_eigenvalues}.
(Appendix~\ref{section:eigenproof}
gives another example, where the eigenvalue decay is allowed
to vary with $n$; in that case, $\Sigma_n$ is benign iff the decay
rate gets close---but not too close---to $1/k$ as $n$ increases.)
The other way to resolve the tension is to consider
a finite dimensional setting (which ensures that the eigenvalues
are summable), and in this case arbitrarily slow decay is possible.
Part~\ref{eigenexample:expplusconst} of
Theorem~\ref{theorem:benign_eigenvalues} gives an example of this:
eigenvalues that are all at least as large as a small constant.
Appendix~\ref{section:eigenproof}
gives another example, with a truncated infinite
series that decays sufficiently slowly that their sum
does not converge.
Theorem~\ref{theorem:benign_eigenvalues}(\ref{eigenexample:infdim})
shows that a very specific decay rate is required in infinite
dimensions, which suggests that this is an unusual phenomenon in that
case. The more generic scenario where benign overfitting will occur is
demonstrated by
Theorem~\ref{theorem:benign_eigenvalues}(\ref{eigenexample:expplusconst}),
with eigenvalues that are either constant or slowly decaying in a very
high---but finite dimensional---space.

\section{Deep neural networks}\label{section:deepnets}

How relevant are Theorems~\ref{th::main}
and~\ref{theorem:benign_eigenvalues} to the phenomenon of benign
overfitting in deep neural networks?  One connection appears by
considering regimes where deep neural networks are well-approximated
by linear functions of their parameters. This so-called {\em
neural tangent kernel} (NTK) viewpoint has been vigorously pursued
recently in an attempt to understand the optimization properties
of deep learning methods.  Very wide neural networks, trained
with gradient descent from a suitable random initialization,
can be accurately approximated by linear functions in an
appropriate Hilbert space, and in this case gradient descent
finds an interpolating solution quickly; see~\cite{ll-lonnsgdsd-18,%
dpzs-gdpoopnn-18,dllwz-gdfgmdnn-18,zou2018stochastic,jgh-ntkcgnn-18,%
adhlw-fgaogotlnn-19}.
(Note that these papers do not consider prediction
accuracy, except when there is no noise; for
example,~\cite[Assumption~A1]{ll-lonnsgdsd-18} implies that the
network can compute a suitable real-valued response exactly, and
the data-dependent bound of~\cite[Theorem~5.1]{adhlw-fgaogotlnn-19}
becomes vacuous when independent noise is added to the $y_i$s.) The
eigenvalues of the covariance operator in this case can have a
heavy tail under reasonable assumptions on the data distribution
(see~\cite{xls-dnnlttf-16}, where this kernel was introduced,
and \cite{cao2019towards}),
and the dimension is very large but
finite, as required for benign overfitting. 
However, the assumptions
of Theorem~\ref{th::main} do not apply in this case. In particular,
the assumption that the random elements of the Hilbert space are
a linearly transformed vector with independent components is not
satisfied. Thus, our results are not directly applicable in
this---somewhat unrealistic---setting.
Note that the slow decay of the eigenvalues of the NTK is in contrast
to the case of the gaussian and other smooth kernels, where the
eigenvalues decay nearly exponentially
quickly~\cite{DBLP:conf/colt/Belkin18}.

The phenomenon of benign overfitting was first observed
in deep neural networks.
Theorems~\ref{th::main} and~\ref{theorem:benign_eigenvalues}
are steps towards understanding this phenomenon by
characterizing when it occurs in the simple setting of linear
regression. Those results suggest that covariance eigenvalues that are
constant or slowly decaying in a high (but finite)
dimensional space might be important in the deep network setting also.
Some authors have suggested viewing neural networks as
finite-dimensional approximations to infinite dimensional objects~\cite{lbw-ealnnbf-96,blvdm-cnn-06,b-bcdcnn-17}, and there are
generalization bounds---although not for the overfitting regime---that
are applicable to infinite width deep networks with parameter norm
constraints~\cite{b-scpcnn-98,bm-rgcrbsr-02,nts-nbccnn-15,snmbnn-manng-17,%
grs-siscnn-18}.
However, the intuition from the linear setting suggests
that truncating to a finite dimensional space might be important
for good statistical performance in the overfitting regime.
Confirming this conjecture by extending our results to the setting
of prediction in deep neural networks is an important open problem.

\section{Proof}\label{section:proof}

Throughout the proofs, we treat $\sigma_x$ (the subgaussian norm of
the covariates) as a constant. Therefore, we use the symbols
$b,c,c_1,c_2,\ldots$ to refer to constants that only depend on
$\sigma_x$. Their values are suitably large (and always at
least $1$) but do not depend on any parameters of the problems we
consider, besides $\sigma_x$. For universal constants that do not depend on any parameters of the problem at all we use the symbol $a$. Also, whenever we sum over eigenvectors of 
$\Sigma$, the sum is restricted to eigenvectors with non-zero
eigenvalues.

\subsection*{Outline}

The first step is a standard decomposition of the excess risk into
two pieces, a term that corresponds to the distortion
that is introduced by viewing $\theta^*$ through the lens of the
finite sample and a term that corresponds to the
distortion introduced by the noise $\bs{\eps}=\bs{y}-X\theta$.
The impact of both sources of error in $\hat\theta$ on
the excess risk is modulated by the covariance $\Sigma$, which
gives different weight to different directions in parameter space.

\begin{lemma}\label{lemma:bv}
  The excess risk of the minimum norm estimator satisfies
    \begin{align*}
      R(\hat\theta)
        \leq 2{\theta^\ast}^\top B \theta^\ast
        + c\sigma^2\log\frac{1}{\delta} \tr(C)
    \end{align*}
  with probability at least $1-\delta$ over $\epsilon$,
  and
    \begin{align*}
      \Expect_{\bs{\eps}} R(\hat\theta)
        & \geq {\theta^\ast}^\top B \theta^\ast + \sigma^2 \tr(C),
    \end{align*}
  where
    \begin{align*}
      B & = \left(I - X^\top \left(X X^\top \right)^{-1}X\right)
          \Sigma\left(I - X^\top \left(X X^\top \right)^{-1}X\right), \\
      C & = \left(X X^\top \right)^{-1}X\Sigma X^\top \left(X X^\top \right)^{-1}.
    \end{align*}
\end{lemma}

The proof of this lemma is in Appendix~\ref{appendix:bias}.
Appendices~\ref{a:upperB} and~\ref{a:packing} give bounds on the term
${\theta^\ast}^\top B \theta^\ast$.
The heart of the proof is controlling $\tr(C)$.

Before continuing with the proof, let us make a quick digression to
note that Lemma~\ref{lemma:bv} already begins to give an idea that
many low-variance directions are necessary for the least-norm
interpolator to succeed.  Let us consider the extreme case that $p =
n$ and $\Sigma = I$.  In this case, $C = \left(X X^\top \right)^{-1}$.
For gaussian data, for instance, standard random matrix theory implies
that, with high probability, the eigenvalues of $X X^\top$ will
all be within a constant factor of $n$,
which implies $\tr(C)$ is bounded below by a constant,
and then Lemma~\ref{lemma:bv} implies that the least-norm interpolant's
excess risk is at least a constant.

To prove that
$\tr(C)$ {\em can}\/ be controlled for suitable
$\Sigma$, the first step is to express it in terms of
sums of outer products of unit covariance independent subgaussian
random vectors.
We show that when there is a $k^*$ with $k^*/n$ small and
$r_{k^*}(\Sigma)/n$ large, all of the
smallest eigenvalues of these matrices are suitably concentrated,
and this implies that $\tr(C)$ is bounded above by
  \[
    \min_{l\le k^*}\left(\frac{l}{n} + 
       n\frac{\sum_{i>l}\lambda_i^2}
      {\left(\lambda_{k^*+1}r_{k^*}(\Sigma)\right)^2}\right).
  \]
(Later, we show that the minimizer is $l=k^*$.)
Next, we show that this expression is also a lower bound on $\tr(C)$,
provided that there is such a $k^*$.
Conversely, we show that for any
$k$ for which $r_k(\Sigma)$ is not large compared to $n$,
$\tr(C)$ is at least as big as a constant times $\min(1,k/n)$.
Combining shows that when $k^*/n$ is small,
$\tr(C)$ is upper and lower bounded by constant factors times
  \[
    \frac{k^*}{n} + \frac{n}{R_{k^*}(\Sigma)}.
  \]

\subsection*{Unit variance subgaussians}

Our assumptions allow the trace of $C$ to be expressed
as a function of many independent subgaussian vectors.

\begin{lemma}\label{lemma::representation_through_z}
  Consider a covariance operator $\Sigma$ with
  $\lambda_i=\mu_i(\Sigma)$ and $\lambda_n>0$. Write its spectral
  decomposition $\Sigma=\sum_{j} \lambda_j v_jv_j^\top$,
  where the orthonormal $v_j\in\bbH$ are the eigenvectors
  corresponding to the $\lambda_j$. 
  For $i$ with $\lambda_i > 0$, define
  $z_i=Xv_i/\sqrt{\lambda_i}$. Then 
    \begin{align*}
      \tr\left(C\right)
        &= \sum_{i}\left[ \lambda_i^2 z_i^\top
          \left(\sum_{j} \lambda_j z_j z_j^\top \right)^{-2}z_i\right],
    \end{align*}
  and these $z_i\in\Re^n$ are independent $\sigma_x^2$-subgaussian.
  Furthermore, for any $i$ with $\lambda_i > 0$, we have
  \begin{align*}
    \lambda_i^2 z_i^\top \left(\sum_{j} \lambda_j z_j z_j^\top
        \right)^{-2}z_i
      = \frac{\lambda_i^2 z_i^\top  A_{-i}^{-2}z_i}
        {(1 + \lambda_i z_i^\top A_{-i}^{-1}z_i)^2},
  \end{align*}
  where $A_{-i} = \sum_{j \neq i} \lambda_j z_j z_j^\top$.
\end{lemma}

\begin{proof}
  By Assumption~\ref{assumption:subgaussiandata} in
  Definition~\ref{def:linreg}, the random
  variables $x^\top v_i/\sqrt{\lambda_i}$ are independent
  $\sigma_x^2$-subgaussian. We consider $X$ in the basis of
  eigenvectors of $\Sigma$, $Xv_i=\sqrt{\lambda_i}z_i$, to see that
  \[
    XX^\top = \sum_i\lambda_i z_i z_i^\top,\qquad
    X\Sigma X^\top = \sum_i\lambda_i^2 z_i z_i^\top,
  \]
  so we can write
    \begin{align*}
      \tr\left(C\right)
        & = \tr\left(\left(X X^\top \right)^{-1}X\Sigma X^\top \left(X
          X^\top \right)^{-1}\right) \\
        & = \sum_{i}\left[ \lambda_i^2 z_i^\top 
          \left(\sum_{j} \lambda_j z_j z_j^\top
          \right)^{-2}z_i\right].
    \end{align*}
  For the second part, we use
  Lemma~\ref{lemma::SMW_formula}, which is
  a consequence of the Sherman-Woodbury-Morrison formula; see
  Appendix~\ref{appendix:SMWproof}.
    \begin{align*}
      \lambda_i^2 z_i^\top  \left(\sum_{j}
        \lambda_j z_j z_j^\top \right)^{-2}z_i
      &= \lambda_i^2 z_i^\top \left(\lambda_i z_i z_i^\top 
        +A_{-i}\right)^{-2}z_i \\
      &= \frac{\lambda_i^2 z_i^\top  A_{-i}^{-2}z_i}
        {(1 + \lambda_i z_i^\top A_{-i}^{-1}z_i)^2},
    \end{align*}
  by Lemma \ref{lemma::SMW_formula}, for the case $k=1$ and
  $Z=\sqrt{\lambda_i}z_i$. Note that $A_{-i}$ is invertible by
  Assumption~\ref{assumption:full_rank} in Definition~\ref{def:linreg}.
\end{proof}

The weighted sum of outer products of these subgaussian vectors plays a
central role in the rest of the proof. Define
  \begin{align*}
    A & = \sum_{i} \lambda_i z_i z_i^\top, &
    A_{-i} & = \sum_{j \neq i} \lambda_j z_j z_j^\top, &
    A_k &= \sum_{i>k}\lambda_i z_i z_i^\top,
  \end{align*}
where the $z_i\in\Re^n$ are independent vectors with independent $\sigma_x^2$-subgaussian  coordinates with unit variance, defined in Lemma~\ref{lemma::representation_through_z}.
Note that the vector $z_i$ is independent of the matrix $A_{-i}$, so
in the last part of Lemma~\ref{lemma::representation_through_z},
all the random quadratic forms are independent of the points where
those forms are evaluated.

\subsection*{Concentration of $A$}

The next step is to show that eigenvalues of $A$, $A_{-i}$ and $A_k$ are
concentrated. The proof of the following inequality is in
Appendix~\ref{appendix:concentration}. Recall that $\mu_1(A)$ and $\mu_n(A)$
denote the largest and the smallest eigenvalues of the $n\times n$ matrix $A$.

\begin{lemma}\label{lemma::eig_bound}
 There is a constant $c$ such that for any $k \geq 0$ 
with probability at least $1 - 2e^{-n/c}$,
    \[
      \frac{1}{c}\sum_{i > k} \lambda_i - c\lambda_{k+1} n
        \leq \mu_n(A_k) \leq \mu_1(A_k)
        \leq c\left(\sum_{i> k} \lambda_i + \lambda_{k+1} n\right).
    \]
\end{lemma}

The following lemma uses this result to give bounds on the eigenvalues
of $A_k$, which in turn give bounds on some eigenvalues of $A_{-i}$
and $A$. For these upper and lower bounds to match up to a constant
factor, the sum of the eigenvalues of $A_k$ should dominate the term
involving its leading eigenvalue, which is a condition on the
effective rank $r_k(\Sigma)$.
The lemma shows that once $r_k(\Sigma)$ is
sufficiently large, all of the eigenvalues of $A_k$ are identical up
to a constant factor.

\begin{lemma}\label{lemma::eigvals_of_truncated}
  There are constants $b,c\ge 1$ such that for any
$k\ge 0$, 
with probability at least $1 - 2e^{-n/c}$,
    \begin{enumerate}
      \item for all $i\ge 1$,
        \begin{align*}
       \mu_{k+1}(A_{-i}) \le \mu_{{k+1}}(A) \le \mu_1(A_k)
           \le c\left(\sum_{j>k}\lambda_j + \lambda_{k+1}n\right),
        \end{align*}
      \item for all $1\le i\le k$,
        \[
          \mu_n(A) \ge \mu_n(A_{-i}) \ge \mu_n\left(A_k\right)
          \ge \frac{1}{c}\sum_{j>k} \lambda_j - c\lambda_{k+1} n,
        \]
      \item if $r_k(\Sigma)\ge bn$, then
        \begin{align*}
          \frac{1}{c}\lambda_{k+1} r_k(\Sigma)
            \le \mu_n\left(A_k\right)
            \le \mu_1(A_k) \le c\lambda_{k+1} r_k(\Sigma).
    \end{align*}
    \end{enumerate}
\end{lemma}
\begin{proof}
  By Lemma~\ref{lemma::eig_bound}, we know that with probability at
  least 
$1-2e^{-n/c_1}$,
    \begin{align*}
      \frac{1}{c_1}\sum_{j>k}\lambda_j - c_1\lambda_{k+1} n
        & \leq \mu_n(A_k) \\
        & \leq \mu_1(A_k) 
        \leq c_1\left(\sum_{j>k} \lambda_j + \lambda_{k+1} n\right).
    \end{align*}
  First, the matrix $A - A_{k}$ has rank at most $k$
  (as a sum of $k$ matrices of rank $1$).
  Thus, there is a linear space $\Ls$ of dimension $n-k$ such that 
  for all $v\in\Ls$, $v^\top Av = v^\top A_k v \leq \mu_1(A_k)\|v\|^2$,
  and so $\mu_{k+1}(A) \le \mu_1(A_k)$.
  
  Second, by the Courant-Fischer-Weyl Theorem, for all $i$ and $j$,
  $\mu_j(A_{-i})\le \mu_j(A)$ (see Lemma~\ref{lemma:monotone}).
  On the other hand, for $i \leq k$, $A_k \preceq A_{-i}$, so all
  the eigenvalues of $ A_{-i}$ are lower bounded by $\mu_n(A_k)$.

  Finally, if $r_k(\Sigma)\ge b n$,
    \begin{align*}
      \sum_{j>k}\lambda_j + \lambda_{k+1}n
        & = \lambda_{k+1} r_k(\Sigma) + \lambda_{k+1}n \\
        & \le \left(1+\frac{1}{b}\right)\lambda_{k+1} r_k(\Sigma), \\
      \frac{1}{c_1}\sum_{j>k}\lambda_j - c_1\lambda_{k+1}n
        & = \frac{1}{c_1}\lambda_{k+1} r_k(\Sigma) - c_1\lambda_{k+1}n
        \\
        & \ge \left(\frac{1}{c_1}-\frac{c_1}{b}\right)
          \lambda_{k+1} r_k(\Sigma).
    \end{align*}
  Choosing $b > c_1^2$ and $c > \max\left\{c_1+1/c_1,
  \left(1/c_1-c_1/b\right)^{-1}\right\}$  gives the
  third
  claim of the lemma.
\end{proof}

\subsection*{Upper bound on the trace term}

\begin{lemma}\label{lemma:traceupper}
  There are constants $b,c\ge 1$ such that 
  if $0\le k\le n/c$, $r_k(\Sigma)\ge bn$, and $l\le k$ 
then  with probability at least 
$1 - 7e^{-n/c}$,
    \begin{align*}
      \tr(C)
        &\le c \left(\frac{l}{n} + n \frac{\sum_{i>l}\lambda_i^2}
            {\left(\sum_{i>k}\lambda_i\right)^2} \right).
    \end{align*}
\end{lemma}
The proof uses the following lemma and its corollary. Their proofs are in
Appendix~\ref{appendix:concentration}.
\begin{lemma}\label{lemma::cor-Bernstein}
  	Suppose $\{\lambda_i\}_{i}^\infty $ is a non-increasing sequence of non-negative numbers such that $\sum_{i=1}^\infty \lambda_i < \infty$, and $\{\xi_i\}_{i = 1}^\infty$ are independent centered $\sigma$-subexponential random variables. Then for some universal constant $a$ for any $t > 0$ with probability at least $1-2e^{-t}$
  	\[
  	\left|\sum_{i} \lambda_i \xi_i\right| \leq a\sigma\max\left(t\lambda_1, \sqrt{t \sum_{i} \lambda_i^2}\right).
  	\]
\end{lemma}
\begin{cor}
	\label{cor::cor norm of projection}
	Suppose $z\in \Re^n$ is a centered random vector with independent
    $\sigma^2$-subgaussian coordinates with unit variances,
    $\Ls$ is a random subspace of $\Re^n$ of codimension
    $k$, and $\Ls$ is independent of $z$. Then for some universal
    constant $a$ and any $t > 0$, with probability at least
    $1-3e^{-t}$,
	\begin{align*}
	\|z\|^2 &\leq n + a\sigma^2(t + \sqrt{nt}),\\
	\|\Pi_{\Ls} z\|^2 &\geq n - a\sigma^2(k + t + \sqrt{nt}),
	\end{align*}
	where $\Pi_{\Ls}$ is the orthogonal projection on $\Ls$.
\end{cor}

\begin{proof}{\em (of Lemma~\ref{lemma:traceupper})}
    Fix $b$ to its value in Lemma~\ref{lemma::eigvals_of_truncated}.
    By Lemma~\ref{lemma::representation_through_z},
    \begin{align}
      \tr(C)
          &= \sum_{i} \lambda_i^2 z_i^\top A^{-2}z_i \notag\\
          &= \sum_{i = 1}^l\frac{\lambda_i^2 z_i^\top  A_{-i}^{-2}z_i}
            {(1 + \lambda_i z_i^\top A_{-i}^{-1}z_i)^2} +
            \sum_{i >l}\lambda_i^2 z_i^\top A^{-2}z_i.
                    \label{eqn:trCupperterms}
    \end{align}
    First, consider the sum up to $l$.
    If $r_k(\Sigma)\ge bn$,
    Lemma~\ref{lemma::eigvals_of_truncated} shows
    that with probability at least 
$1-2e^{-n/{c_1}}$,
    for all $i\le k$, $\mu_{n}(A_{-i})\ge \lambda_{k+1}r_k(\Sigma)/c_1$,
    and, for all $i$, $\mu_{k+1}(A_{-i})\le c_1\lambda_{k+1}r_k(\Sigma)$.
    The lower bounds on the $\mu_{n}(A_{-i})$'s imply that,
    for all $z \in \R^n$ and $1\le i \leq l$,
      \begin{align*}
        z^\top A_{-i}^{-2} z
          &\leq
          \frac{c_1^2\|z\|^2}{\left(\lambda_{k+1}r_k(\Sigma)\right)^2},
      \end{align*}
    and the upper bounds on the $\mu_{k+1}(A_{-i})$'s give
      \begin{align*}
        z^\top A_{-i}^{-1} z
          &\geq \left(\Pi_{\Ls_i} z\right)^\top
            A_{-i}^{-1} \Pi_{\Ls_i} z 
          \geq\frac{\|\Pi_{\Ls_i} z\|^2}{c_1\lambda_{k+1}r_k(\Sigma)}, &
      \end{align*}
    where 
    $\Ls_i$ is the span of the
    $n-k$ eigenvectors of $A_{-i}$ corresponding to its smallest $n-k$
    eigenvalues. So for $i\le l$,
      \begin{align}
        \frac{\lambda_i^2 z_i^\top  A_{-i}^{-2}z_i}{(1 + \lambda_i z_i^\top A_{-i}^{-1}z_i)^2}
          & \leq \frac{z_i^\top  A_{-i}^{-2}z_i}{(z_i^\top A_{-i}^{-1}z_i)^2}
          \leq c_1^4\frac{\|z_i\|^2}{\|\Pi_{\Ls_i}z_i\|^4}.
        \label{eq::bound_on_Cterm}
      \end{align}
    Next, we apply Corollary~\ref{cor::cor norm of projection} $l$ times, together with a union bound,
    to show that with probability at least $1-3e^{-t}$, for all $1\le i\le l$,
      \begin{align}
        \|z_i\|^2
          & \leq n + a\sigma_x^2(t + \ln k + \sqrt{n(t + \ln k)})\le c_2n, \label{eqn:zinequality1} \\
        \|\Pi_{\Ls_i}z_i\|^2
          & \geq n-a\sigma_x^2(k + t + \ln k + \sqrt{n(t + \ln k)}) \ge n/c_3,        \label{eqn:zinequality2}
      \end{align}
    provided that  $t < n/c_0$ and $c > c_0$ for some sufficiently
    large $c_0$ (note that $c_2$ and $c_3$ only depend on $c_0$, $a$
    and $\sigma_x$, and we can still take $c$ large enough in the end
    without changing $c_2$ and $c_3$).
    Combining~\eqref{eq::bound_on_Cterm},~\eqref{eqn:zinequality1},
    and~\eqref{eqn:zinequality2}, with probability at least
   $1-5e^{-n/c_0}$,
      \begin{align*}
        \sum_{i=1}^l\frac{\lambda_i^2 z_i^\top
            A_{-i}^{-2}z_i}{(1 + \lambda_i z_i^\top A_{-i}^{-1}z_i)^2}
          \leq c_4\frac{l}{n}.
      \end{align*}
    Second, consider the second sum in~\eqref{eqn:trCupperterms}.
    Lemma~\ref{lemma::eigvals_of_truncated} shows that, on the same high
    probability event that we considered in bounding the first half of the sum,
    $\mu_n(A)\ge\lambda_{k+1}r_k(\Sigma)/c_1$. Hence,
      \begin{align*}
        \sum_{i>l}\lambda_i^2 z_i^\top A^{-2}z_i
          &\le \frac{c_1^2\sum_{i>l}\lambda_i^2\|z_i\|^2}
            {\left(\lambda_{k+1}r_k(\Sigma)\right)^2}.
      \end{align*}
    Notice that $\sum_{i>l}\lambda_i^2\|z_i\|^2$ is a weighted sum of $\sigma_x^2$-subexponential random
    variables, with the weights given by the $\lambda_i^2$ in blocks of size $n$.
    Lemma~\ref{lemma::cor-Bernstein} implies that, with probability at least $1-2e^{-t}$,
      \begin{align*}
        \sum_{i>l}\lambda_i^2\|z_i\|^2
          & \le n\sum_{i>l}\lambda_i^2
            + a\sigma_x^2\max\left(\lambda_{l+1}^2t,
            \sqrt{tn\sum_{i>l}\lambda_i^4}\right) \\
          & \le n\sum_{i>l}\lambda_i^2
            + a\sigma_x^2\max\left(t\sum_{i>l}\lambda_i^2,
            \sqrt{tn}\sum_{i>l}\lambda_i^2\right) \\
          & \le c_5n\sum_{i>l}\lambda_i^2,
      \end{align*}
    because $t < n/c_0$. Combining the above gives
      \begin{align*}
        \sum_{i>l}\lambda_i^2 z_i^\top A^{-2}z_i
          &\le c_{6}n\frac{\sum_{i>l}\lambda_i^2}{\left(\lambda_{k+1}r_k(\Sigma)\right)^2}.
      \end{align*}
     Finally, putting both parts together and taking $c > \max\{c_0,
     c_4, c_6\}$ gives the lemma.
\end{proof}

\subsection*{Lower bound on the trace term}

We first give a bound on a single term in the expression for $\tr(C)$
in Lemma~\ref{lemma::representation_through_z}
that holds regardless of $r_k(\Sigma)$. The proof is in
Appendix~\ref{appendix:singletermlower}.

\begin{lemma}\label{lemma:singletermlower}
  There is a constant $c$ such that for any  $i\ge 1$ with $\lambda_i > 0$, and any $0\le k\le n/c$, 
 with probability  at least 
$1-5e^{-n/c}$,
    \begin{align*}
      \frac{\lambda_i^2 z_i^\top A_{-i}^{-2}z_i}
          {(1 + \lambda_i z_i^\top A_{-i}^{-1}z_i)^2}  \ge 
   \frac{1}{cn} \left(1+\frac{\sum_{j >k} \lambda_j
          + n\lambda_{k+1}}{n\lambda_i}\right)^{-2}.
    \end{align*}
\end{lemma}

We can extend these bounds to a lower bound on $\tr(C)$
using the following lemma. The proof is in Appendix~\ref{appendix:sum_of_pos}.

\begin{lemma}\label{lemma::sum_of_pos}
  Suppose $n\leq \infty$ and $\{\eta_i\}_{i = 1}^n$ is a sequence
  of non-negative random variables, $\{t_i\}_{i=1}^n$ is a sequence
  of non-negative real numbers (at least one of which is strictly
  positive) such that for some $\delta \in (0,1)$ and any
  $i \leq n$, $\Pr(\eta_i > t_i) \geq 1 - \delta$. Then
    \[
      \Pr\left(\sum_{i=1}^n\eta_i \geq \frac12\sum_{i=1}^n t_i\right)
        \geq 1 - 2\delta.
    \]
\end{lemma}

These two lemmas imply the following lower bound.
\begin{lemma}\label{lemma:lowerbound}
  There are constants $c$ such that for any $0\le k\le n/c$ and any $ b > 1$
  with probability at least 
$1-10e^{-n/c}$,
  \begin{enumerate}
    \item If $r_k(\Sigma)<bn$, then $\tr(C)\ge\frac{k+1}{c b^2n}$.
    \item If $r_k(\Sigma)\ge bn$, then
        \[
          \tr(C)\ge \frac{1}{cb^2} \min_{l\le k}\left(
          \frac{l}{n} + \frac{b^2n \sum_{i>l} \lambda_i^2}
          {\left(\lambda_{k+1}r_k(\Sigma)\right)^2}\right).
        \]
  \end{enumerate}
  In particular, if all choices of $k\le n/c$ give $r_k(\Sigma)<bn$, then
  $r_{n/c}(\Sigma)<bn$ implies that with probability at least 
$1-10e^{-n/c}$,
  $\tr(C)= \Omega_{\sigma_x}( 1 )$.
\end{lemma}
\begin{proof} 
From Lemmas~\ref{lemma::representation_through_z}, \ref{lemma:singletermlower} and~\ref{lemma::sum_of_pos},
with probability at least 
$1-10e^{-n/c_1}$,
  \begin{align*}
    \tr(C)
      & \ge  \frac{1}{c_1n} \sum_i \left(1+\frac{\sum_{j >k} \lambda_j
        + n\lambda_{k+1}}{n\lambda_i}\right)^{-2} \\
      &\ge \frac{1}{c_2n} \sum_i \min\left\{
        1, \frac{n^2\lambda_i^2}{\left(\sum_{j >k} \lambda_j\right)^2},
        \frac{\lambda_i^2}{\lambda_{k+1}^2}\right\} \\
      &\ge \frac{1}{c_2b^2n}\sum_i \min\left\{
        1, \left(\frac{bn}{r_k(\Sigma)}\right)^2
        \frac{\lambda_i^2}{\lambda_{k+1}^2},
        \frac{\lambda_i^2}{\lambda_{k+1}^2}\right\}.
  \end{align*}
Now, if $r_k(\Sigma)<bn$, then the second term in the minimum
is always bigger than the third term, and in that case,
  \begin{align*}
    \tr(C)
      &\ge\frac{1}{c_2b^2n} \sum_i \min\left\{
        1, \frac{\lambda_i^2}{\lambda_{k+1}^2}\right\}
      \ge \frac{k+1}{c_2b^2n}.
  \end{align*}
On the other hand, if $r_k(\lambda)\ge bn$,
  \begin{align*}
    \tr(C)
      &\ge\frac{1}{c_2b^2} \sum_i \min\left\{
        \frac{1}{n},
        \frac{b^2n\lambda_i^2}
        {\left(\lambda_{k+1}r_k(\Sigma)\right)^2}\right\} \\
      &= \frac{1}{c_2b^2} \min_{l\le k}\left(
        \frac{l}{n} + \frac{b^2n \sum_{i>l} \lambda_i^2}
        {\left(\lambda_{k+1}r_k(\Sigma)\right)^2}\right),
  \end{align*}
where the equality follows from the fact that the $\lambda_i$s are
non-increasing.
\end{proof}

\subsection*{A simple choice of $l$}

Recall that $\sigma_x$ is a constant. 
If no $k\le n/c$ has $r_k(\Sigma)\ge bn$,
then Lemmas~\ref{lemma:bv} and \ref{lemma:lowerbound}
imply that the expected excess risk is $\Omega(\sigma^2)$,
which proves the first paragraph of Theorem~\ref{th::main} for large
$k^*$.
If some $k\le n/c$ does have $r_k(\Sigma)\ge bn$, then the upper and lower
bounds of Lemmas~\ref{lemma:traceupper} and~\ref{lemma:lowerbound} are
constant multiples of
  \[
    \min_{l\le k} \left(
      \frac{l}{n} + n\frac{\sum_{i>l}\lambda_i^2}
      {\left(\lambda_{k+1}r_k(\Sigma)\right)^2}\right).
  \]
It might seem surprising that any suitable choice of $k$ suffices
to give upper and lower bounds: what prevents one choice of $k$
from giving an upper bound that falls below the lower bound that
arises from another choice of $k$? However, the freedom to choose
$k$ is somewhat illusory: Lemma~\ref{lemma::eigvals_of_truncated}
shows that, for any qualifying value of $k$, the smallest eigenvalue
of $A$ is within a constant factor of $\lambda_{k+1}r_k(\Sigma)$.
Thus, any two choices of $k$ satisfying $k\le n/c$ and $r_k(\Sigma)\ge
bn$ must have values of $\lambda_{k+1}r_k(\Sigma)$ within constant
factors.  The smallest such $k$ simplifies the bound on $\tr(C)$, as
the following lemma shows. The proof is in Appendix~\ref{appendix:bestk}.

\begin{lemma}\label{lemma:bestk}
For any $b\ge 1$ and $k^* := \min\left\{k: r_k(\Sigma)\ge bn\right\}$,
if $k^* < \infty$, we have
  \begin{align*}
    \min_{l\le k^*} \left(
      \frac{l}{bn} + \frac{bn\sum_{i>l}\lambda_i^2}
      {\left(\lambda_{k^*+1}r_{k^*}(\Sigma)\right)^2}\right)
%    & \\
%     & 
    = \frac{k^*}{bn} + \frac{bn\sum_{i>k^*}\lambda_i^2}
      {\left(\lambda_{k^*+1}r_{k^*}(\Sigma)\right)^2}
    = \frac{k^*}{bn} + \frac{bn}{R_{k^*}(\Sigma)}.
  \end{align*}
\end{lemma}

Finally, we can finish the proof of Theorem~\ref{th::main}. 
Set $b$ in Lemma~\ref{lemma:lowerbound} and Theorem~\ref{th::main} to the constant $b$ from  Lemma~\ref{lemma:traceupper}. Take $c_1$ to be the maximum of the constants $c$ from  Lemmas~\ref{lemma:lowerbound} and~\ref{lemma:traceupper}.

By Lemma~\ref{lemma:lowerbound}, if $k^* \geq n/c_1$, then
w.h.p.~$\tr(C)\geq 1/c_2$.
However, by the second part of Lemma~\ref{lemma:lowerbound} and
by Lemma~\ref{lemma:traceupper}, if $k^* < n/c_1$, then w.h.p.~$\tr(C)$
is within a constant factor of $\min_{l\le k^*} \left(
      \frac{l}{n} + n\frac{\sum_{i>l}\lambda_i^2}
      {\left(\lambda_{k^*+1}r_{k^*}(\Sigma)\right)^2}\right)$,
which, by Lemma~\ref{lemma:bestk}, is 
within a constant factor of $\frac{k^*}{n} + \frac{n}{R_{k^*}(\Sigma)}.$
Taking $c$ sufficiently large, and combining
these results with Lemma~\ref{lemma:bv} and with the upper bound on the
term ${\theta^\ast}^\top B \theta^\ast$ in Appendix~\ref{a:upperB}
completes the proof of the first paragraph of Theorem~\ref{th::main}.

The proof of the second paragraph is in Appendix~\ref{a:packing}.

\section{Conclusions and Further Work}\label{section:further}

Our results characterize when the phenomenon of benign overfitting
occurs in high dimensional linear regression, with gaussian data
and more generally. We
give finite sample excess risk bounds that reveal the covariance
structure that ensures that the minimum norm interpolating prediction
rule has near-optimal prediction accuracy. The characterization
depends on two notions of the effective rank of the data covariance
operator. It shows that overparameterization, that is, the existence
of many low-variance and hence unimportant directions in parameter
space, is essential for benign overfitting, and that
data that lies in a large but finite dimensional space exhibits the
benign overfitting phenomenon with a much wider range of covariance
properties than data that lies in an infinite dimensional space.

There are several natural future directions. Our main theorem
requires the conditional expectation $\Expect[y|x]$ to be a linear
function of $x$, and it is important to understand whether the results
are also true in the misspecified setting, where this assumption is
not true. Our main result also assumes
that the covariates are distributed as a linear function of a vector
of independent random variables. We would like to understand the
extent to which this assumption can be relaxed, since it rules out
some important examples, such as infinite-dimensional reproducing
kernel Hilbert spaces with continuous kernels defined on
finite-dimensional spaces.
We would also like to
understand how our results extend to other loss functions besides
squared error and what we can say about overfitting estimators
beyond the minimum norm interpolating estimator. The most interesting
future direction is understanding how these ideas could apply to
nonlinearly parameterized function classes such as neural networks,
the methodology that uncovered the phenomenon of benign overfitting.

\section*{Acknowledgements}

We gratefully acknowledge the support of the NSF through grant
IIS-1619362 and of Google through a Google Research Award.  Part of
this work was done as part the Fall 2018 program on Foundations of
Data Science at the Simons Institute for the Theory of Computing.
G\'abor Lugosi was supported by the Spanish Ministry of Economy and
Competitiveness, Grant MTM2015-67304-P and FEDER, EU;
``High-dimensional problems in structured probabilistic models -
Ayudas Fundaci\'on BBVA a Equipos de Investigaci\'on Cientifica
2017''; and Google Focused Award ``Algorithms and Learning for AI''.

% Bibliography
\bibliographystyle{plain}
\bibliography{bibliography}

\appendix

\section{Proof of Lemma~\ref{lemma:bv}}\label{appendix:bias}

We first give the decomposition of the excess risk.

\begin{lemma}
\label{l:bias.variance}
  The excess risk of the minimum norm estimator satisfies
    \begin{align*}
      R(\hat\theta)
        & = \E_{x} \left(x^\top\left(\theta^*-\hat\theta\right)\right)^2
        \leq 2{\theta^\ast}^\top B \theta^\ast +2\bs{\eps}^\top
        C\bs{\eps},
    \end{align*}
  and
    \begin{align*}
      \Expect_{x,\bs{\eps}} R(\hat\theta)
        & \geq {\theta^\ast}^\top B \theta^\ast + \sigma^2 \tr(C),
    \end{align*}
  where
    \begin{align*}
      B & = \left(I - X^\top \left(X X^\top \right)^{-1}X\right)
          \Sigma\left(I - X^\top \left(X X^\top \right)^{-1}X\right), \\
      C & = \left(X X^\top \right)^{-1}X\Sigma X^\top \left(X X^\top \right)^{-1}.
    \end{align*}
\end{lemma}

\begin{proof}
  Since $\eps=y-x^\top\theta^*$ has mean zero conditionally on $x$,
    \begin{align*}
      R(\hat\theta)
        & = \Expect_{x,y}\left(y-x^\top\hat\theta\right)^2
          - \Expect\left(y-x^\top\theta^*\right)^2 \\
        & = \Expect_{x,y}\left(y-x^\top\theta^*
          + x^\top\left(\theta^*-\hat\theta\right)\right)^2
          - \Expect\left(y-x^\top\theta^*\right)^2 \\
        & = \E_{x} \left(x^\top\left(\theta^*-\hat\theta\right)\right)^2.
    \end{align*}
  Using~\eqref{equation:hattheta}, the definition of $\Sigma$, and the
  fact that $\bs{y}=X\theta^*+\bs{\eps}$,
    \begin{align*}
      R(\hat\theta)
        & = \E_{x} \left(x^\top\left(I - X^\top \left(X X^\top \right)^{-1}
          X\right)\theta^\ast - x^\top X^\top \left(X X^\top \right)^{-1}
          \bs{\eps}\right)^2 \\
        &\leq 2 \E_{x} \left(x^\top\left(I - X^\top \left(X X^\top \right)^{-1}
          X\right)\theta^\ast\right)^2 
          + 2 \E_{x} \left(x^\top X^\top \left(X X^\top
          \right)^{-1}\bs{\eps}\right)^2\\
        & = 2{\theta^\ast}^\top \left(I - X^\top \left(X X^\top \right)^{-1}X\right)
          \Sigma\left(I - X^\top \left(X X^\top \right)^{-1}X\right) \theta^\ast \\
       & \hspace{0.5in}
          +2\bs{\eps}^\top \left(X X^\top \right)^{-1}X\Sigma X^\top
          \left(X X^\top \right)^{-1}\bs{\eps} \\
        & = 2{\theta^\ast}^\top B \theta^\ast +2\bs{\eps}^\top
        C\bs{\eps}.
  \end{align*}
Also, since $\bs{\eps}$ has zero mean conditionally on $X$, and is independent of $x$, we have
    \begin{align*}
      \Expect_{x,\bs{\eps}} R(\hat\theta)
        & = \E_{x,\bs{\eps}}\left[
          \left(x^\top\left(I - X^\top \left(X X^\top \right)^{-1}
          X\right)\theta^\ast\right)^2 
          + \left(x^\top X^\top \left(X X^\top \right)^{-1}\bs{\eps}\right)^2\right]\\
        & = {\theta^\ast}^\top \left(I - X^\top \left(X X^\top \right)^{-1}X\right)
          \Sigma\left(I - X^\top \left(X X^\top \right)^{-1}X\right) \theta^\ast \\
       & \hspace{0.5in}
          +\tr\left(\left(X X^\top \right)^{-1}X\Sigma X^\top \left(X X^\top \right)^{-1}
          \Expect\left[\bs{\eps}\bs{\eps}^\top\middle|X\right]\right) \\
        & \geq {\theta^\ast}^\top B \theta^\ast + \sigma^2 \tr\left(C\right).
    \end{align*}
\end{proof}

The following lemma shows that we can obtain a high-probability upper
bound on the term $\bs{\eps}^\top C\bs{\eps}$ in terms
of the trace of $C$. It is Lemma~36 in~\cite{Page2017}.
\begin{lemma}
	\label{lemma::subgaussian-quadratic-form}
  Consider random variables $\eps_1,\ldots,\eps_n$, conditionally
  independent given $X$ and conditionally $\sigma^2$-subgaussian, that
  is, for all $\lambda\in\Re$,
    \[
      \E[\exp(\lambda\eps_i)|X] \leq \exp(\sigma^2\lambda^2/2).
    \]
  Suppose that, given $X$, $M\in\Re^{n\times n}$ is a.s.~positive
  semidefinite. Then a.s.~on $X$, with conditional probability at
  least $1-e^{-t}$,
    \[
      \bs{\eps}^\top M \bs{\eps}
        \le \sigma^2\tr(M) + 2\sigma^2 \|M\|t
          + 2\sigma^2\sqrt{\|M\|^2 t^2 +
          \tr\left(M^2\right)t}.
    \]
\end{lemma}
Since $\|C\|\leq \tr(C)$ and $\tr(C^2) \leq \tr(C)^2$,
with probability at least $1-e^{-t}$,
  \begin{align*}
    \bs{\eps}^\top C\bs{\eps}
      & \leq \sigma^2\tr(C)(2t+1) + 2\sigma^2\sqrt{\tr(C)^2(t^2 + t)}
      \leq (4t+2)\sigma^2\tr(C).
  \end{align*}
Combining this with Lemma~\ref{l:bias.variance}
implies Lemma~\ref{lemma:bv}.

\section{An Algebraic Property}\label{appendix:SMWproof}.

\begin{lemma}\label{lemma::SMW_formula}
  Suppose $k < n$, $A \in \R^{n\times n}$ is an invertible matrix,
  and $Z \in \R^{n\times k}$ is such that $ZZ^\top + A$ is invertible.
  Then
    \[
      Z^\top (ZZ^\top + A)^{-2} Z
      = (I +  Z^\top A^{-1}Z)^{-1}Z^\top A^{-2}Z
        (I +  Z^\top A^{-1}Z)^{-1}.
     \]
\end{lemma}

\begin{proof}
  We use the Sherman–Morrison–Woodbury formula to write
  \begin{equation}\label{equation:SMW}
    (ZZ^\top + A)^{-1}
    = A^{-1} - A^{-1}Z(I + Z^\top A^{-1}Z)^{-1}Z^\top A^{-1}.
  \end{equation}
  Denote $M_1:= Z^\top A^{-1}Z$ and $M_2:= Z^\top A^{-2}Z$.
  Applying~\eqref{equation:SMW}, we get
	\begin{align*}
	Z^\top(ZZ^\top + A)^{-2} Z 
	& = Z^\top \Bigl(A^{-1} - A^{-1}Z(I + Z^\top A^{-1}Z)^{-1}Z^\top A^{-1}\Bigr)^2Z \\
	& = Z^\top \Bigl(A^{-1} - A^{-1}Z(I + M_1)^{-1}Z^\top A^{-1}\Bigr)^2Z \\
	& = Z^\top \Bigl(
                 A^{-2}
                 - A^{-2}Z(I + M_1)^{-1}Z^\top A^{-1} 
                 - A^{-1}Z(I + M_1)^{-1}Z^\top A^{-2}
                  \\
        & \hspace{0.5in}
          + A^{-1}Z(I + M_1)^{-1}Z^\top A^{-2}Z(I + M_1)^{-1}Z^\top A^{-1}
                \Bigr)Z \\
	& = M_2 - M_2(I + M_1)^{-1}M_1 
      - M_1(I + M_1)^{-1}M_2 \\
        & \hspace{0.5in}
          + M_1(I + M_1)^{-1}M_2(I + M_1)^{-1}M_1\\
	& = M_2 - M_2(I + M_1)^{-1}M_1 
      - M_1(I + M_1)^{-1}M_2 (I - (I + M_1)^{-1}M_1) \\
	& = M_2(I + M_1)^{-1} -  M_1(I + M_1)^{-1}M_2(I + M_1)^{-1}\\
	& = (I + M_1)^{-1}M_2(I + M_1)^{-1},
	\end{align*}
  where we used the identity $I - (I + M_1)^{-1}M_1 = (I + M_1)^{-1}$
  twice in the second last equality and the identity
  $I - M_1(I + M_1)^{-1} = (I + M_1)^{-1}$ in the last equality.
\end{proof}

\section{Proof of concentration inequalities}\label{appendix:concentration}

We use some standard results about subgaussian and
subexponential random variables.

First of all, we need the following direct consequence of
Propositions~2.5.2 and~2.7.1 and Lemma~2.7.6
from~\cite{vershynin_2018}:
\begin{lemma}\label{lemma:square of subgaussian}
    There is a universal constant $c$ such that for any
    random variable $\xi$ that is centered, $\sigma^2$-subgaussian,
    and unit variance, $\xi^2 - 1$ is a centered
    $c\sigma^2$-subexponential random variable, that is,
	  \[
		\E\exp(\lambda(\xi^2 - 1)) \leq
        \exp(c^2\sigma^4\lambda^2)\text{ for all such $\lambda$ that }
        |\lambda| \leq \frac{1}{c\sigma^2}.
	\]
\end{lemma}

Second, we are going to use the following form of Bernstein's inequality, which is Theorem 2.8.2 in \cite{vershynin_2018}:
\begin{lemma}\label{lemma:Bernstein}
    There is a universal constant $c$ such that, for any
	independent, mean zero, $\sigma$-subexponential random variables
	$\xi_1, \dots, \xi_N$, any $a = (a_1, \dots, a_N)\in \Re^n$, and
    any $t \geq 0$,
	  \begin{align*}
	    \Pbb\left(\left|\sum_{i=1}^N a_i \xi_i\right| > t\right)
            \leq 2\exp\left[-c\min\left(\frac{t^2}{\sigma^2 \sum_{i=1}^N
        a_i^2}, \frac{t}{\sigma \max_{1 \leq i \leq n}
        a_i}\right)\right].
	  \end{align*}
\end{lemma}
\begin{cor}\label{cor:Bernstein}
    There is a universal constant $c$ such that for any
	non-increasing sequence $\{\lambda_i\}_{i =1}^\infty $
    of non-negative numbers such that
    $\sum_{i=1}^\infty \lambda_i < \infty$, and any
    independent, centered, $\sigma$-subexponential random variables
    $\{\xi_i\}_{i = 1}^\infty$, and any $x > 0$, with probability
    at least $1-2e^{-x}$
	  \[
	    \left|\sum_{i} \lambda_i \xi_i\right| \leq
        c\sigma\max\left(x\lambda_1, \sqrt{x \sum_{i}
        \lambda_i^2}\right).
	  \]
\end{cor}
\begin{cor}\label{cor:norm of projection}
    There is a universal constant $c$ such that for any
	centered random vector $z\in \Re^n$ with independent
    $\sigma^2$-subgaussian coordinates with unit variances, any
    random subspace $\Ls$ of $\Re^n$ of codimension $k$ that is
    independent of $z$, and any $t > 0$, with probability at
    least $1-3e^{-t}$,
	  \begin{align*}
	    \|z\|^2 \leq& n + c\sigma^2(t + \sqrt{nt}),\\
	    \|\Pi_{\Ls} z\|^2 \geq& n - c\sigma^2(k + t + \sqrt{nt}),
	  \end{align*}
	where $\Pi_{\Ls}$ is the orthogonal projection on $\Ls$.
\end{cor}
\begin{proof}
	First of all, since $\|z\|^2 = \sum_{i=1}^n z_i^2$ --- a sum of
    $n$ $\sigma^2$-subexponential random variables, by
    Corollary~\ref{cor:Bernstein}, for some absolute constant $c$
    and for any $t > 0$, with probability at least $1 - 2e^{-t}$,
	\[
		\left|\|z\|^2 - n\right| \leq c\sigma^2\max(t, \sqrt{nt}).
	\]
	
	Second, we can write
	\[
		\|\Pi_{\Ls} z\|^2 = \|z\|^2 - \|\Pi_{\Ls^\perp} z\|^2.
	\]
	
	Denote $M = \Pi_{\Ls^\perp}^\top \Pi_{\Ls^\perp}$. 	Since $\|M\| =
    1$ and $\tr\left(M\right) = \tr(M^2) = k$, by
    Lemma~\ref{lemma::subgaussian-quadratic-form},
    with probability at least $1 - e^{-t}$,
	\begin{align*}
	 \|\Pi_{\Ls^\perp} z\|^2 =& z^\top M z\\
	 \leq& \sigma^2 k + 2\sigma^2 t + 2\sigma^2\sqrt{t^2 + kt}\\
	 \leq& \sigma^2(2k+4t).
	\end{align*}
	
	Thus, with probability at least $1-3e^{-t}$
	\begin{align*}
	\|z\|^2 \leq& n + c\sigma^2\max(t, \sqrt{nt}),\\
	\|\Pi_{\Ls} z\|^2 \geq& \|z\| -  \sigma^2(2k+4t)\\
	\geq& n -  \sigma^2(2k + 4t + c\max(t, \sqrt{nt})).
	\end{align*}
\end{proof}

\begin{lemma}[$\epsilon$-net argument]\label{lemma::norm_by_net}
  Suppose $A \in \R^{n\times n}$ is a symmetric matrix, and
  $\Nc_\epsilon$ is an $\epsilon$-net on the unit sphere $\Sc^{n-1}$ in
  the Euclidean norm, where $\epsilon < \frac12$. Then
    \[
      \|A\| \leq (1 - \epsilon)^{-2} \max_{x \in \Nc_\epsilon} |x^\top Ax|.
    \]
\end{lemma}
\begin{proof}
  Denote the eigenvalues of $A$ as $\lambda_1, \dots, \lambda_n$
  and assume $|\lambda_1| \geq |\lambda_2| \geq \dots \geq |\lambda_n|$.
  Denote the first eigenvector of $A$ as $v \in \Sc^{n-1}$, and take
  $\Delta v \in \R^n$ such that $v + \Delta v \in \Nc_\epsilon$ and
  $\|\Delta v\| \leq \epsilon.$ 
  Denote the coordinates of $\Delta v$ in
  the eigenbasis of $A$ as $\Delta v_1, \dots, \Delta v_n$.
  Now we can write
    \begin{align*}
      \left|(v + \Delta v)^\top  A (v + \Delta v)\right|
        &= \left|\lambda_1 + 2\lambda_1\Delta v_1
          + \sum_{i=1}^n \lambda_i \Delta v_i^2\right|\\
        &= |\lambda_1|\cdot \left|1 + 2\Delta v_1 + \Delta v_1^2
          + \sum_{i=2}^n \frac{\lambda_i}{\lambda_1} \Delta v_i^2\right|\\
        &\geq|\lambda_1|\cdot \left|1 + 2\Delta v_1 + \Delta v_1^2
          - \sum_{i=2}^n \Delta v_i^2\right|\\
        &= |\lambda_1|\cdot \left|1 + 2\Delta v_1 + \Delta v_1^2 - \|\Delta v\|^2 + \Delta v_1^2\right|\\
        &= |\lambda_1|\cdot \left|1 + 2\left(\Delta v_1 + \Delta v_1^2\right) - \|\Delta v\|^2 \right|\\
        &\ge |\lambda_1|\cdot \left|1 + 2\left(-\|\Delta v\| + (-\|\Delta v\|)^2\right) - \|\Delta v\|^2 \right|\\
        &=|\lambda_1|\cdot \left|1 - 2\|\Delta v\| + \|\Delta v\|^2 \right|\\
        &\geq |\lambda_1| \cdot |1 - 2\epsilon + \epsilon^2|\\
        &= \|A\|(1-\epsilon)^2,
    \end{align*}
  where the first inequality holds because the $\lambda_i$s are
  decreasing in magnitude, and the last two inequalities hold since
  the functions $x + x^2$ and $2x + x^2$ are both increasing on
  $(-\frac12, \infty)$ and $\Delta v_1 \geq -\|\Delta v\| \geq
  -\epsilon \geq -\frac12$.
\end{proof}

We restate Lemma~\ref{lemma::eig_bound}.

\begin{lemma}
  There is a universal constant $c$ such that with probability at least $1 - 2e^{-n/c}$,
    \[
      \frac{1}{c}\sum_{i} \lambda_i - c\lambda_1 n
        \leq \mu_n(A) \leq \mu_1(A)
        \leq c\left(\sum_{i} \lambda_i + \lambda_1 n\right).
    \]
\end{lemma}
\begin{proof}
  For a fixed vector $v \in \R^n$, Proposition~2.6.1 from
  \cite{vershynin_2018} implies that for some constant $c_1$ and any
  $i$ the random variable $v^\top z_i$ is
  $c_1\|v\|^2\sigma_x^2$-subgaussian. Thus, for any fixed unit vector
  $v$, as $v^\top A v = \sum_{i} \lambda_i(v^\top z_i)^2$,
  Lemma~\ref{lemma:square of subgaussian} and
  Corollary~\ref{cor:Bernstein} imply that for some constant $c_2$ with
  probability at least $1 - 2e^{-t}$,
    \[
      \left|v^\top A v - \sum \lambda_i\right|
        \leq c_2\sigma_x^2\max\left(\lambda_1 t, \sqrt{t\sum \lambda_i^2}\right).
    \]
  Let $\Nc$ be a $\frac14$-net on the sphere $\Sc^{n-1}$ with respect to the
  Euclidean distance such that $|\Nc|\leq 9^n$.
  Applying the union bound over the elements of $\Nc$, we see that
  with probability $1 - 2e^{-t}$, every $v \in \Nc$ satisfies
   \begin{align*}
    &  \left|v^\top A v - \sum \lambda_i\right| 
        \leq c_2\sigma_x^2\max\left(\lambda_1 (t + n\ln 9), \sqrt{(t + n\ln 9)\sum_{i} \lambda_i^2}\right).
   \end{align*}
  Since $\Nc$ is a $\frac14$-net, by Lemma~\ref{lemma::norm_by_net},
  we need to multiply the quantity above by $(1 - 1/4)^{-2}$ to
  get the bound on the norm of the $A - I_n\sum_{i} \lambda_i$. Denote
  \[
  \Diamond = \left(\lambda_1 (t + n\ln 9)
  + \sqrt{(t + n\ln 9)\sum_{i} \lambda_i^2}\right).
  \]
  Thus, with probability at least $1 - 2e^{-t}$,
  \[
  \left\|A - I_n\sum_i \lambda_i\right\| \leq c_3\sigma_x^2\Diamond.
  \]
  When $t<n/c_4$ we can write $t+n\ln 9\le c_5n$, and we have
    \begin{align*}
      \Diamond
        & \le c_5\left(\lambda_1 n + \sqrt{n\sum_i \lambda_i^2}\right) \\
        & \le c_5\lambda_1 n
          + \sqrt{\left(c_5^2\lambda_1 n\right)
          \sum_i \lambda_i} \\
        & \le c_6\sigma_x^2\lambda_1 n + \frac{1}{2c_3\sigma_x^2}\sum_i \lambda_i,
    \end{align*}
  by the AMGM inequality. (Recall that $c_1,c_2,\ldots$ denote universal
  constants with value at least $1$, and $\sigma_x \geq 1/c_7$ is the
  subgaussian constant of a random variable with unit variance.)
\end{proof}

\section{Proof of Lemma~\ref{lemma:singletermlower}}
\label{appendix:singletermlower}

  Fix $i\ge 1$ with $\lambda_i > 0$
  and $0\le k\le n/c$. By Lemma~\ref{lemma::eigvals_of_truncated},
  with probability at least 
$1-2e^{-n/c_1}$,
    \[
      \mu_{k+1}(A_{-i})
      \le c_1\left(\sum_{j>k} \lambda_j + \lambda_{k+1}n\right),
    \]
  and hence
    \begin{align*}
      z_i^\top A_{-i}^{-1}z_i
        &\geq \frac{\|\Pi_{\Ls_i} z_i\|^2}
          {c_1\left( \sum_{j >k} \lambda_j + \lambda_{k+1}n\right)}.
    \end{align*}
  By Corollary~\ref{cor::cor norm of projection},
  with probability at least $1-3e^{-t}$,
    \[
      \|\Pi_{\Ls_i} z_i\|^2 \geq n-a\sigma_x^2(k + t + \sqrt{tn})\ge n/c_2,
    \]
    provided that $t < n/c_0$ and $c > c_0$ for some sufficiently large $c_0$.
  Thus, with probability at least 
$1-5e^{-n/c_3}$,
    \begin{align*}
      z_i^\top A_{-i}^{-1}z_i
      \geq \frac{ n} { c_3\left(\sum_{j >k} \lambda_j + \lambda_{k+1}n\right)},
    \end{align*}
  hence
    \begin{align*}
      1 + \lambda_i z_i^\top A_{-i}^{-1}z_i 
  \leq \left(\frac{c_3\left(\sum_{j >k} \lambda_j + \lambda_{k+1}n\right)}
          {\lambda_in} + 1\right)
          \lambda_i z_i^\top A_{-i}^{-1}z_i.
    \end{align*}
  Dividing $\lambda_i^2z_i^\top  A_{-i}^{-2}z_i$ by the square of both
  sides, we have
    \begin{align*}
      \frac{\lambda_i^2 z_i^\top  A_{-i}^{-2}z_i}
          {(1 + \lambda_i z_i^\top A_{-i}^{-1}z_i)^2}
        & \geq \left(\frac{c_3\left(\sum_{j >k} \lambda_j + \lambda_{k+1}n\right)}
          {\lambda_in} + 1\right)^{-2}
          \frac{z_i^\top A_{-i}^{-2}z_i}{(z_i^\top A_{-i}^{-1}z_i)^2}.
    \end{align*}
  Also, from the Cauchy-Schwarz inequality 
  and Corollary~\ref{cor::cor norm of projection} again,
  we have that 
  on the same event,
    \begin{align*}
      \frac{ z_i^\top  A_{-i}^{-2}z_i}{( z_i^\top A_{-i}^{-1}z_i)^2}
        &\geq \frac{ z_i^\top
        A_{-i}^{-2}z_i}{\left\|A_{-i}^{-1}z_i\right\|^2\|z_i\|^2} \\
        & = \frac{1}{\|z_i\|^2}
        \geq\frac{1}{n + a\sigma_x^2(t + \sqrt{nt})}
        \geq\frac{1}{c_4 n}.
    \end{align*}
  Choosing $c$ suitably large gives the lemma.

\section{Proof of Lemma~\ref{lemma::sum_of_pos}}
\label{appendix:sum_of_pos}
  We know that, for all $i\le n$,
    $ \Pbb(\eta_i > t_i) \geq 1 - \delta$.
  Consider the following event: 
    \[
      E = \left\{\sum_{i=1}^n \eta_i < \frac12\sum_{i=1}^n
      t_i\right\},
    \]
  and denote its probability as $ c\delta$ for some $c \in (0,\delta^{-1})$. 
  On the one hand, by the definition of the event, we have
    \[
      \frac{1}{\Pbb(E)}\E\left[\ind_E \sum_{i=1}^n \eta_i\right] \leq \frac12\sum_{i=1}^n t_i.
    \]
  On the other hand, note that for any $i$,
    \begin{align*}
      \E[\eta_i \ind_{E}] \geq& \E[ t_i\ind_{\{\eta_i\geq t_i\} \cap E} ]\\
        &= t_i \Pbb(\{\eta_i\geq t_i\} \cap E)\\
        &\geq t_i( \Pbb\{\eta_i\geq t_i\} +\Pbb(E) -1)\\
        &\geq t_i(c-1)\delta.
    \end{align*}
  So
    \begin{align*}
      \E\left[\ind_E \sum_{i=1}^n \eta_i\right]
        & \geq (c-1)\delta\sum_{i=1}^n t_i, \\
      \frac{1}{\Pbb(E)}\E\left[\ind_E \sum_{i=1}^n \eta_i\right]
        & \geq (1-c^{-1})\sum_{i=1}^n t_i.
    \end{align*}
  Thus, we obtain
    \begin{gather*}
      \frac12\sum_{i=1}^n t_i \geq (1-c^{-1})\sum_{i=1}^n t_i, \\
      c \leq 2, \\
      \Pbb\left(\sum_{i=1}^n\eta_i < \frac12\sum_{i=1}^n t_i\right) = c\delta \leq 2\delta.
    \end{gather*}

\section{Proof of Lemma~\ref{lemma:bestk}}
\label{appendix:bestk}

  We can write the function of $l$ being minimized as
    \begin{align*}
      \frac{l}{bn} + \frac{bn\sum_{i>l}\lambda_i^2}
      {\left(\lambda_{k^*+1}r_{k^*}(\Sigma)\right)^2}
        & = \sum_{i=1}^l \frac{1}{bn} +\sum_{i>l}
          \frac{bn\lambda_i^2}
          {\left(\lambda_{k^*+1}r_{k^*}(\Sigma)\right)^2} \\
        & \ge \sum_{i=1}^{k^*} \min\left\{\frac{1}{bn},
          \frac{bn\lambda_i^2}
          {\left(\lambda_{k^*+1}r_{k^*}(\Sigma)\right)^2}\right\} \\*
        & \qquad {}
          + \sum_{i>k^*} \frac{bn\lambda_i^2}
          {\left(\lambda_{k^*+1}r_{k^*}(\Sigma)\right)^2} \\
        & = \sum_{i=1}^{l^*} \frac{1}{bn} +\sum_{i>l^*}
         \frac{bn\lambda_i^2}
        {\left(\lambda_{k^*+1}r_{k^*}(\Sigma)\right)^2},
    \end{align*}
  where $l^*$ is the largest value of $i\le k^*$ for which 
    \[
      \frac{1}{bn} \le
         \frac{bn\lambda_i^2}
        {\left(\lambda_{k^*+1}r_{k^*}(\Sigma)\right)^2},
    \]
  since the $\lambda_i^2$ are non-increasing.
  This condition 
  holds iff
    \begin{align*}
      \lambda_i
        \ge \frac{\lambda_{k^*+1}r_{k^*}(\Sigma)}{bn}.
    \end{align*}
  The definition of $k^*$ implies $r_{k^*-1}(\Sigma)<bn$. So we can write
    \begin{align*}
      r_{k^*}(\Sigma)
        & = \frac{\sum_{i>k^*}\lambda_i}{\lambda_{k^*+1}} \\
        & = \frac{\sum_{i>k^*-1}\lambda_i-\lambda_{k^*}}{\lambda_{k^*+1}} \\
        & = \frac{\lambda_{k^*}}{\lambda_{k^*+1}} (r_{k^*-1}(\Sigma) -1) \\
        & < \frac{\lambda_{k^*}}{\lambda_{k^*+1}}(bn-1),
    \end{align*}
  and so
  the minimizing $l$ is $k^*$. Also,
    \begin{align*}
      \frac{\sum_{i>k^*}\lambda_i^2} {\left(\lambda_{k^*+1}r_{k^*}(\Sigma)\right)^2}
        & = \frac{\sum_{i>k^*}\lambda_i^2}{\left(\sum_{i>k^*}\lambda_i\right)^2}
        = \frac{1}{R_{k^*}(\Sigma)}.
    \end{align*}

\section{Eigenvalue monotonicity}

Recall (half of) the Courant-Fischer-Weyl theorem.
\begin{lemma}
\label{lemma:CFT}
For any symmetric $n \times n$ matrix $A$, and any $i \in [n]$, 
$\mu_i(A)$ is 
the minimum, over all subspaces $U$ of $\R^n$
of dimension $n - i$, of the maximum, over all unit-length
$u \in U$, of $u^\top A u$.
\end{lemma}

\begin{lemma}[Monotonicity of eigenvalues]\label{lemma:monotone}
If symmetric matrices $A$ and $B$
satisfy $A\preceq B$, then, for any $i \in [n]$, we have
$\mu_i(A)\le\mu_i(B)$.
\end{lemma}
\begin{proof}
Let $U$ be the subspace of $\R^n$ of dimension $n-i$
that minimizes the maximum
over all unit-length $u \in U$, of $u^\top A u$, and
let $V$ be the analogous subspace for $B$.
We have
\begin{align*}
\mu_i(A) 
 & = \max_{u \in U: || u || = 1} u^\top A u 
          \;\;\;\mbox{(by Lemma~\ref{lemma:CFT})}
            \\
 & \leq \max_{v \in V: || v || = 1} v^\top A v 
       \;\;\;\mbox{(since $U$ is the minimizer)} \\
 & \leq \max_{v \in V: || v || = 1} v^\top B v 
       \;\;\;\mbox{(since $A \preceq B$)} \\
 & = \mu_i(B),
\end{align*}
by Lemma~\ref{lemma:CFT}, completing the proof.
\end{proof}

\section{Rank facts}\label{appendix:rankfacts}

The quantity $r_0(\Sigma)$ is an important complexity parameter
for covariance estimation problems, where it has been called the
`effective rank'~\cite{v-inaarm-12,koltchinskii2017}.  Earlier,
$r_0(\Sigma^2)$ was called the `stable rank'~\cite{rv-hwisgc-13}
and the `numerical rank'~\cite{rv-sflmagfa-07}, although that
term has a different meaning in computational linear
algebra~\cite[p261]{gvl-mc-13}.

We restate Lemma~\ref{lemma:rkRkordering}.
\begin{lemma}
  $r_k(\Sigma)\ge 1$, $r^2_k(\Sigma)=r_k(\Sigma^2)R_k(\Sigma)$, and
  $ r_k(\Sigma^2) \le r_k(\Sigma) \le R_k(\Sigma) \le r_k^2(\Sigma)$.
\end{lemma}
\begin{proof}
  The first inequality and the equality are immediate from the
  definitions. Together they imply $R_k(\Sigma) \le r_k^2(\Sigma)$.
  For the second inequality,
  \begin{align*}
    r_k(\Sigma^2)
    & = \frac{\sum_{i>k}\lambda_i^2}{\lambda_{k+1}^2}
    \le \frac{\lambda_{k+1}\sum_{i>k}\lambda_i}{\lambda_{k+1}^2}
    = r_{k}(\Sigma).
  \end{align*}
  Substituting this in the equality implies
  $r_k(\Sigma) \le R_k(\Sigma)$.
\end{proof}

\begin{lemma}\label{lemma:monotonicbound}
Writing $r_k$ and $R_k$ for $r_k(\Sigma)$ and $R_k(\Sigma)$,
  \[
    \frac{1}{R_{k+1}} = \frac{\frac{1}{R_k}-\frac{1}{r_k^2}}
      {1-\left(2-\frac{1}{r_k}\right)\frac{1}{r_k}}.
  \]
  Thus, the function $\phi(k)= k/(b^2n) + n/R_k$
  satisfies the monotonicity property $\phi(k+1)>\phi(k)$ whenever
  $r_k>bn\ge 1$.
\end{lemma}

\begin{proof}
Writing
  \begin{align*}
    q = \sum_{i>k+1}\lambda_i^2,\;\; s = \sum_{i>k+1}\lambda_i,
  \end{align*}
so that $R_{k+1}=s^2/q$, we have
  \begin{align*}
      \frac{1}{R_k} - \frac{1}{R_{k+1}}
        & = \frac{\lambda_{k+1}^2 +q}{\left(\lambda_{k+1}+s\right)^2}
          - \frac{q}{s^2} \\
        & = \frac{\left(\lambda_{k+1}^2 +q\right)s^2 - q\left(\lambda_{k+1}+s\right)^2}
          {s^2\left(\lambda_{k+1}+s\right)^2} \\
        & = \frac{1}{r_k^2} - \frac{q\lambda_{k+1}\left(\lambda_{k+1}+2s\right)}
          {s^2\left(\lambda_{k+1}+s\right)^2} \\
        & = \frac{1}{r_k^2} - \frac{2\left(\lambda_{k+1}+s\right)-\lambda_{k+1}}
          {R_{k+1}r_k \left(\lambda_{k+1}+s\right)} \\
        & = \frac{1}{r_k^2} - \frac{2-1/r_k}{R_{k+1}r_k}.
  \end{align*}
Hence
  \begin{align*}
    \frac{1}{R_{k+1}}
      & = \frac{1/R_k - 1/r_k^2}
        {1-\left(2-\frac{1}{r_k}\right)\frac{1}{r_k}}.
  \end{align*}
Since
$r_k>1$, $0<1-\left(2-1/r_k\right)/r_k<1$, so
  \begin{align*}
      \frac{n}{R_{k+1}}
        & > \frac{n}{R_k} - \frac{n}{r^2_k},
  \end{align*}
and if $r_k>bn$,
  \begin{align*}
    \phi(k+1)-\phi(k)
      & = \frac{k+1}{b^2n} + \frac{n}{R_{k+1}}
        - \left(\frac{k}{b^2n} + \frac{n}{R_k}\right) \\
      & > \frac{1}{b^2n} - \frac{n}{r^2_k}\\
      & >0.
  \end{align*}
\end{proof}

\section{Conditions on eigenvalues}\label{section:eigenproof}

In this section, we prove the following expanded version of
Theorem~\ref{theorem:benign_eigenvalues}.
  \begin{theorem}\label{theorem:benign_eigenvalues_expanded}
    Define
     $\lambda_{k,n}:= \mu_k(\Sigma_n)$ for all $k,n$.
    \begin{enumerate}
      \item\label{eigenexample_expanded:infdim}
        If $\lambda_{k,n} = k^{-\alpha} \ln^{-\beta} (k+1)$, then
        $\Sigma_n$ is benign iff $\alpha=1$ and
        $\beta > 1$.
      \item\label{eigenexample_expanded:exponent}
        If $\lambda_{k,n} = k^{-(1+\alpha_n)}$, then
        $\Sigma_n$ is benign iff $\omega(1/n)=\alpha_n=o(1)$.
        Furthermore,
          \[
            R(\hat\theta)
              = \Theta\left(\min\left\{\frac{1}{\alpha_nn}
              +\alpha_n,1\right\}\right).
          \]
      \item\label{eigenexample_expanded:finitepoly}
        If
          \[
            \lambda_{k,n} = \begin{cases}
              k^{-\alpha} & \text{if $k\le p_n$,} \\
              0   & \text{otherwise,}
              \end{cases}
          \]
        then $\Sigma_n$ is benign iff either $0<\alpha<1$,
        $p_n=\omega(n)$ and $p_n=o\left(n^{1/(1-\alpha)}\right)$ or
        $\alpha=1$, $p_n = e^{\omega(\sqrt{n})}$ and  $p_n = e^{o(n)}$.
      \item\label{eigenexample_expanded:expplusconst}
        If
          \[
            \lambda_{k,n} = \begin{cases}
              \gamma_k + \epsilon_n & \text{if $k\le p_n$,} \\
              0   & \text{otherwise,}
              \end{cases}
          \]
        and $\gamma_k=\Theta(\exp(-k/\tau))$,
        then $\Sigma_n$ is benign iff $p_n=\omega(n)$ and
        $ne^{-o(n)}=\epsilon_np_n=o(n)$.
        Furthermore, for $p_n=\Omega(n)$ and
        $\epsilon_np_n=ne^{-o(n)}$,
          \[
            R(\hat\theta)
              = O\left(\frac{\epsilon_np_n+1}{n}
              + \frac{\ln(n/(\epsilon_np_n))}{n}
              + \max\left\{\frac{1}{n},\frac{n}{p_n}\right\}\right).
          \]
    \end{enumerate}
  \end{theorem}

We build up the proof in stages.  First, we characterize those
sequences of effective ranks that can arise.

	\begin{theorem}
		Consider some positive summable sequence $\{\lambda_i\}_{i=1}^\infty$, and for any non-negative integer $i$ denote 
		\[
		r_i := \lambda_{i+1}^{-1}\sum_{j > i}\lambda_j.
		\]
		Then $r_i > 1$ and
		$ \sum_{i} r_i^{-1} = \infty$.
		Moreover, for any positive sequence $\{u_i\}$ such
                that  $\sum_{i = 0}^\infty u_i^{-1} = \infty$ and for
                every $i$ $u_i > 1$, there exists a positive sequence $\{\lambda_i\}$ (unique up to
                constant multiplier) such that  $r_i \equiv u_i$.  The sequence is (a constant rescaling of)
			 \[
			 \lambda_k = u_{k-1}^{-1}\prod_{i=0}^{k-2}(1 - u_i^{-1}).
			 \]
	\end{theorem}
	\begin{proof}
			\begin{align*}
			\sum_{i \geq k+1} \lambda_i =&  \sum_{i \geq k} \lambda_i - \lambda_k
			= (1 - r_{k-1}^{-1})\sum_{i \geq k} \lambda_i.
			\end{align*}
			Thus, 
			\[
			\sum_{i \geq k+1} \lambda_i = \prod_{i =
            0}^{k-1} \left(1 - r_{i}^{-1} \right) \cdot \sum_{i} \lambda_i,
			\]
			which goes to zero if and only if $\sum_{i} r_i^{-1} = \infty$.
			 On the other hand, we may rewrite the first equality in the proof as
			\[
				 \lambda_{k+1}r_k = \lambda_kr_{k-1}(1 - r_{k-1}^{-1}),
			\]
            and hence
			\[
				\lambda_kr_{k-1} =
                \prod_{i=0}^{k-2}\left(1-r_i^{-1}\right)
                \lambda_1r_0.
			\]
			So for any sequence $\{u_i\}$ we can uniquely (up to a
            constant multiplier) recover the sequence $\{\lambda_i\}$ such that  $r_i = u_i$ --- the only candidate is 
			 \[
			 \lambda_k = u_{k-1}^{-1}\prod_{i=0}^{k-2}(1 - u_i^{-1}).
			 \]
			 However, for such $\{\lambda_i\}$ one can compute
			 \[
			 \sum_{i=1}^k \lambda_i = 1 - \prod_{i=0}^{k-1}(1 - u_i^{-1}),
			 \]
			 so the resulting sequence $\{\lambda_i\}$  sums to 1, and 
			 \[
			 r_k = \lambda_{k+1}^{-1}\sum_{i > k} \lambda_i = \lambda_{k+1}^{-1}\prod_{i=0}^{k-1}(1 - u_i^{-1}) = u_k.
			 \]
 			 
 	\end{proof}

	\begin{theorem}
      \label{t:k.by.r}
		Suppose $b$ is some constant, and $k^\ast(n) =
                \min\{k: r_k \geq bn\}$. Suppose also that the
                sequence $\{r_n\}$ is increasing. Then, as $n$ goes to
                infinity,  $k^\ast(n)/n$ goes to zero  if and only if
                $r_n/n$ goes to infinity.
		\end{theorem}
	\begin{proof}
		We prove the ``if'' part separately from the ``only if'' part.
		\begin{enumerate}
			\item {\bf If  $k^\ast(n)/n \to 0$ then $r_n/n \to \infty$.}
			
			Fix some $C > 1$.  Since $k^\ast(n)/n \to 0$, there exists some $N_C$ such that  for any $n \geq N_C$, $k^\ast(n) < n/C.$ Thus,  for all $n > N_C$, 
			\begin{gather*}
			k^\ast(\lfloor Cn\rfloor) \leq n,\\
			 r_n \geq  r_{k^\ast(\lfloor Cn\rfloor)} \geq b\lfloor Cn\rfloor.
			 \end{gather*}
			 
			 Since the constant $C$ is arbitrary, $r_n/n$ goes to infinity.
		 
		 	\item {\bf If $r_n/n \to \infty$ then  $k^\ast(n)/n \to 0$ .}
		 	
		 	Fix some constant $C > 1$. Since $r_n/n \to \infty$  there exists some $N_C$ such that  for any $n \geq N_C$, $r_n > Cn$. Thus,  for any $n > CN_C/b$ 
		 	\begin{gather*}
			 	r_{\lceil nb/C\rceil } \geq bn,\\
			 	k^\ast(n) \leq \lceil nb/C\rceil.
			\end{gather*} 
		 	 Since the constant $C$ is arbitrary, $k^\ast(n)/n$ goes to zero.
		 \end{enumerate}
	\end{proof}

	\begin{theorem}
          \label{t:third_term_suff}
		Suppose the sequence $\{r_i\}$ is increasing and $r_n/n \to \infty$ as $n \to \infty$. Then a sufficient condition for $	\frac{n}{R_{k^\ast(n)}} \to 0$ is 
		\[
		r_k^{-2} = o(r_k^{-1} - r_{k+1}^{-1}) \text{ as } k \to \infty.
		\]
		For example, this condition holds for $r_n = n\log n$.
	\end{theorem}
	\begin{proof}
		We need to show that 
		\[
		\frac{n}{R_{k^\ast(n)}} = \frac{n\sum_{i > k^\ast(n)}\lambda_i^2}{\left(\sum_{i > k^\ast(n)}\lambda_i\right)^2} = \frac{n\sum_{i > k^\ast(n)}\lambda_i^2}{\lambda_{k^\ast(n)+1}^2r_{k^\ast(n)}^2} \to 0.
		\]
		Since $r_{k^\ast(n)} \geq bn$
                and $\lim_{n \rightarrow \infty} k^*(n) = \infty$, 
                it is enough to prove
                that $\frac{\sum_{i >
                    k}\lambda_i^2}{\lambda_{k+1}^2r_{k}} \to 0$ as $k$
                goes to infinity.
                Since 
		\[
		\lambda_{k+2} r_{k+1} =  \lambda_{k+1}r_k(1 - r_k^{-1}),
		\]
		we can write that
                \begin{align*}
		 \lambda_{k+1+l}r_{k+l} & =  \lambda_{k+1}r_k\prod_{i = k}^{k+l-1} (1-r_{i}^{-1}) \\
               & \leq \lambda_{k+1}r_k \exp\left(-\sum_{i=k}^{k+l-1}r_{i}^{-1}\right)
                \end{align*}
                which yields
		\[
                \frac{\lambda_{k+1+l}}{\lambda_{k+1}r_k}
		 \leq r_{k+l}^{-1} \exp\left(-\sum_{i=k}^{k+l-1}r_{i}^{-1}\right).
		\]
		Thus, we obtain
		\[
		\frac{\sum_{i > k}\lambda_i^2}{\lambda_{k+1}^2r_{k}}\leq
		r_k\sum_{i \geq k} r_{i}^{-2}\exp\left(-2\sum_{j=k}^{i-1}r_{j}^{-1}\right),
		\]
		and it is sufficient to prove that the latter quantity goes to zero. We write
		\begin{align*}
		 r_k\sum_{i \geq k} r_{i}^{-2}\exp\left(-2\sum_{j=k}^{i-1}r_{j}^{-1}\right) =& \frac{\sum_{i \geq k} r_{i}^{-2}\exp\left(-2\sum_{j=k}^{i-1}r_{j}^{-1}\right)}{r_k^{-1}}\\
		 =&\frac{\sum_{i \geq k} r_{i}^{-2}\exp\left(-2\sum_{j=0}^{i-1}r_{j}^{-1}\right)}{r_k^{-1}\exp\left(-2\sum_{j=0}^{k-1}r_{j}^{-1}\right)}.
		 \end{align*}
		 Since both numerator and denominator are decreasing in $k$ and go to zero as $k \to \infty$, we can apply the Stolz–Ces\'{a}ro theorem (an analog of L'H\^{o}pital's rule for discrete sequences):
		\begin{align*}
                \lim_{k \rightarrow \infty}
                \frac{\sum_{i \geq k}
                r_{i}^{-2}\exp\left(-2\sum_{j=0}^{i-1}r_{j}^{-1}\right)}{r_k^{-1}\exp\left(-2\sum_{j=0}^{k-1}r_{j}^{-1}\right)}
               & =  \lim_{k \rightarrow \infty}
		\frac{r_{k}^{-2}\exp\left(-2\sum_{j=0}^{k-1}r_{j}^{-1}\right)}{(r_k^{-1} - e^{-2r_k^{-1}}r_{k+1}^{-1})\exp\left(-2\sum_{j=0}^{k-1}r_{j}^{-1}\right) } \\
                 & =  \lim_{k \rightarrow \infty} 
           \frac{r_{k}^{-2}}{(r_k^{-1} - e^{-2r_k^{-1}}r_{k+1}^{-1})} \\
 & \hspace{0.2in}
   \quad 
   \quad 
              \text{(since, for large enough $k$,
                  $e^{-2r_k^{-1}} \leq 1-r_k^{-1}$)} \\
		& \leq 
          \lim_{k \rightarrow \infty} \frac{r_{k}^{-2}}{r_k^{-1} - r_{k+1}^{-1} + r_k^{-1}r_{k+1}^{-1}}\\
		& = 0,
		\end{align*}
		where the last line is due to our sufficient condition.
	\end{proof}

\medskip
Now we are ready to prove
Theorem~\ref{theorem:benign_eigenvalues_expanded}.

\noindent
{\bf Part \ref{eigenexample_expanded:infdim}, if direction, first term:}  We have
\[
|| \Sigma_n || \sqrt{r_0(\Sigma_n)} 
  = \sqrt{ \lambda_1 \sum_{i=1}^{\infty} \lambda_i}
      = O\left( \sqrt{\sum_{i=1}^{\infty} \frac{1}{i \log^{\beta} (1 +
      i)}} \right),
\]
which is $O(1)$ for $\beta > 1$.

\noindent
{\bf Part \ref{eigenexample_expanded:infdim}, if direction, second term:}  
By Theorem~\ref{t:k.by.r}, it suffices to prove
that $\lim_{n\rightarrow \infty} \frac{r_n}{n} = \infty$.  This holds because
\begin{align*}
r_n 
 & = \frac{ \sum_{i > n} \frac{1}{i \log^{\beta} (1 + i)}}{\frac{1}{(n+1) \log^{\beta} (2 + n)}} 
  = \Theta(n \log n),
\end{align*}
since $\beta > 1$.

\noindent
{\bf Part \ref{eigenexample_expanded:infdim}, if direction, third term:}  By Theorem~\ref{t:third_term_suff}, it suffices to
prove that $r_k^{-2} = o(r_k^{-1} - r_{k+1}^{-1})$, that is
\begin{align*}
& \lim_{k \rightarrow \infty} 
 \frac{r_k^{-2}}{r_k^{-1} - r_{k+1}^{-1}} = 0 
\end{align*}
or, equivalently,
\begin{align*}
    \lim_{k \rightarrow \infty} 
         \frac{r_{k+1}}{r_k (r_{k+1} - r_k)} = 0.
\end{align*}
As argued above, when $\alpha = 1$ and $\beta > 1$, $r_k = \Theta(k
\log k)$, so it suffices to show that $\lim_{k \rightarrow \infty}
(r_{k+1} - r_k) = \infty$.  We have
\begin{align*}
r_{k+1} - r_k 
 & = \frac{\sum_{i > k+1} \lambda_i}{\lambda_{k+2}}
      - \frac{\sum_{i > k} \lambda_i}{\lambda_{k+1}} \\
 & = \frac{\left( ( \lambda_{k+1} - \lambda_{k+2}) \sum_{i > k+1} \lambda_i \right)
           - \lambda_{k+1}\lambda_{k+2}}
          {\lambda_{k+1}\lambda_{k+2}}
       \\
 & = \left( \left( \frac{1}{\lambda_{k+2}} - \frac{1}{\lambda_{k+1}} \right)
         \sum_{i > k+1} \lambda_i \right)
      - 1 \\
\end{align*}
so it suffices to show that
\[
\lim_{k \rightarrow \infty}  \left( \frac{1}{\lambda_{k+2}} - \frac{1}{\lambda_{k+1}} \right)
         \sum_{i > k+1} \lambda_i = \infty.
\]
Since $\lambda_i$ is non-increasing, we have
\begin{align*}
\left( \frac{1}{\lambda_{k+2}} - \frac{1}{\lambda_{k+1}} \right)
         \sum_{i > k+1} \lambda_i 
& \geq \left( \frac{1}{\lambda_{k+2}} - \frac{1}{\lambda_{k+1}} \right)
         \int_{k+1}^{\infty} \frac{1}{x \log^{\beta} x} \; dx \\
& = \left( \frac{1}{\lambda_{k+2}} - \frac{1}{\lambda_{k+1}} \right)
         \frac{1}{(\beta - 1) \log^{\beta - 1} (k+1)} \\
& = \frac{(k+2) \log^{\beta} (k+3) - (k+1) \log^{\beta} (k+2)}
         {(\beta - 1) \log^{\beta - 1} (k+1)}.
\end{align*}
If we define $f$ on the positive reals by
$f(x) = x \log^{\beta} (x+1)$, then $f$ is 
convex, and, since 
$f'(x) = \frac{ \beta x \log^{\beta - 1} (x+1)}{x+1} + \log^{\beta} (x+1)$,
we have
\begin{align*}
\frac{(k+2) \log^{\beta} (k+3) - (k+1) \log^{\beta} (k+2)}
         {(\beta - 1) \log^{\beta - 1} (k+1)}
 & \geq 
   \frac{\frac{ \beta (k+1) \log^{\beta - 1} (k+2)}
              {k+2} 
          + \log^{\beta} (k+2)}
        {(\beta - 1) \log^{\beta - 1} (k+1)},
\end{align*}
which goes to infinity for large $k$, completing the proof of the
``if'' direction of the third term of Part~\ref{eigenexample_expanded:infdim}.

\noindent
{\bf Part~\ref{eigenexample_expanded:infdim}, only if direction, $\alpha > 1$:}  
If $\alpha > 1$, then 
\begin{align*}
r_n 
 & = \frac{ \sum_{i > n} \frac{1}{i^a \log^{\beta} (1 + i)}}
          {\frac{1}{n^a \log^{\beta} (1 + n)}} \\
 & \leq n^{\alpha} 
    \sum_{i > n} \frac{\log^{\beta} (1 + n)}{i^a \log^{\beta} (1 + i)} \\
 & \leq n^{\alpha} 
    \sum_{i > n} \frac{1}{i^a} \\
 & = n^{\alpha} O(n^{1 - \alpha}),
\end{align*}
which does not grow faster than $n$.  Thus, by
Theorem~\ref{t:k.by.r}, $k^\ast(n)/n$ does not go to zero.

\noindent
{\bf Part \ref{eigenexample_expanded:infdim}, only if direction, 
   $\alpha < 1$, or $\alpha = 1$ and $\beta \leq 1$:}  
In this case, since, as above
\begin{align*}
|| \Sigma_n || \sqrt{r_0(\Sigma_n)} 
  \geq  \sqrt{\sum_{i=1}^{\infty} \lambda_i},
\end{align*}
and $\sum_{i=1}^{\infty} \frac{1}{i^{\alpha} \log^{\beta} (1 + i)}$
diverges in this case, 
$\frac{|| \Sigma_n || \sqrt{r_0(\Sigma_n)}}{n}$ does not go to zero.

\noindent
Before starting on Part~\ref{eigenexample_expanded:exponent},
let us define $r_{k,n} = r_k(\Sigma_n)$ and $R_{k,n} = R_k(\Sigma_n)$.

\noindent
{\bf Part \ref{eigenexample_expanded:exponent}, if direction, first term:}  We have
\[
|| \Sigma_n || \sqrt{r_{0,n}} 
  = \sqrt{ \lambda_{1,n} \sum_{i=1}^{\infty} \lambda_{i.n}}
      = \sqrt{\sum_{i=1}^{\infty} \frac{1}{i^{1+\alpha_n}}}
      \leq \sqrt{1 + \frac{1}{\alpha_n}},
\]
so $|| \Sigma_n || \sqrt{ \frac{ r_{0,n}}{n}} \leq \sqrt{\frac{1 +
    \frac{1}{\alpha_n}}{n}}$ which goes to zero with $n$ if $\alpha_n
= \omega(1/n)$.

\noindent
{\bf Part \ref{eigenexample_expanded:exponent}, if direction, second term:}  
First, 
\begin{align*}
r_{k,n}
& = (k+1)^{1 + \alpha_n} \sum_{i > k} i^{-(1 + \alpha_n)} \\
& \geq (k+1)^{1 + \alpha_n} \int_{k+1}^{\infty} x^{-(1 + \alpha_n)} dx \\
& = \frac{k+1}{\alpha_n}.
\end{align*}
Thus, $k^*(n) = O(\alpha_n n)$, so that 
$\frac{k^*(n)}{n} = O(\alpha_n) = o(1)$.

\noindent
{\bf Part \ref{eigenexample_expanded:exponent}, if direction, third term:}  
We bound $R_{k,n}$ from below by separately bounding
its numerator and denominator:
\begin{align*}
\sum_{i > k} i^{-(1 + \alpha_n)}
& \geq \int_{k+1}^{\infty} x^{-(1 + \alpha_n)} \;dx \\
& = \frac{1}{\alpha_n (k+1)^{\alpha_n}},
\end{align*}
and
\begin{align*}
\sum_{i > k} i^{-2 (1 + \alpha_n)} 
& \leq \int_k^{\infty} x^{-2 (1 + \alpha_n)} \;dx \\
& = \frac{1}{ k^{1 + 2 \alpha_n} (2 \alpha_n + 1)},
\end{align*}
so that
\begin{equation}
\label{e:R.by.k}
R_{k,n} 
 \geq \frac{k^{1 + 2 \alpha_n} (2 \alpha_n + 1)}{\alpha_n^2 (k+1)^{2 \alpha_n}}
 \geq 
 \frac{k}{\alpha_n^2}
  \times
 \left( 1 - \frac{1}{k+1} \right)^{2 \alpha_n}.
\end{equation}
So now we want a lower bound on $k^*(n)$.  For that, we need
an upper bound on $r_{k,n}$, and
\begin{align*}
r_{k,n}
& \leq (k+1)^{1 + \alpha_n} \int_k^{\infty} x^{-(1 + \alpha_n)} dx \\
& = 
 \frac{(k+1)}{\alpha_n}
  \times
 \left(1 + \frac{1}{k} \right)^{\alpha_n} \\
& \leq 
 \frac{2 k}{\alpha_n} e^{\alpha_n/k}.
\end{align*}
This implies $\frac{2 k^*(n)}{\alpha_n} e^{\alpha_n/k^*(n)} \geq b n$.  
This, together with the fact
that,
for $u > 1$, $u e^{1/u}$ is an increasing function
of $u$, implies that, for large enough $n$,
$k^*(n) \geq \alpha_n b n/3$.  
Since $\alpha_n = \omega(1/n)$, this implies
that $k^*(n) = \omega(1)$.  
Combining this with \eqref{e:R.by.k},
for large enough $n$
\begin{align*}
R_{k^*(n),n} 
 & \geq  \frac{k^*(n)}{\alpha_n^2} e^{-\alpha_n/k^*(n)} 
 \geq  \frac{k^*(n)}{2 \alpha_n^2} 
 \geq  \frac{b n}{6 \alpha_n}.
\end{align*}
Thus $n/R_{k^*(n),n} = O(\alpha_n) = o(1)$.

\noindent
{\bf Part \ref{eigenexample_expanded:exponent}, only if direction, $\alpha_n = O(1/n)$:}  
We have
\begin{align*}
|| \Sigma_n || \sqrt{r_{0,n}} 
 &= \sqrt{\sum_{i=1}^{\infty} \frac{1}{i^{1+\alpha_n}}} 
  \geq \sqrt{\frac{1}{\alpha_n}},
\end{align*}
so $|| \Sigma_n || \sqrt{\frac{r_{0,n}}{n}}
     \geq \sqrt{\frac{1}{\alpha_n n}}$, which is 
bounded below by a constant for large $n$ if 
$\alpha_n = O(1/n)$.

\noindent
{\bf Part \ref{eigenexample_expanded:exponent}, only if direction, $\alpha_n = \Omega(1)$:}  
Recall that, in the proof of the ``if'' direction of
the third term, we showed that $k^*(n) \geq \alpha_n b n/3$.
This implies that $\frac{k^*(n)}{n} = \Omega(\alpha_n)$.

\noindent
{\bf Part~\ref{eigenexample_expanded:finitepoly}:}
Suppose that $\Sigma_n$ is benign. Then because
$R_{k}(\Sigma_n)\le p_n-k$, we must have $p_n=\omega(n)$. Thus, we can restrict our attention to the sequences for which $p_n = \omega(n)$ and find the necessary and sufficient conditions for that class.

Next, for any positive $\alpha$ and any natural number
$k \in [1, p_n)$, we can write 

\begin{equation*}
 \int_{k}^{p_n}x^{-\alpha}\,dx \geq \sum_{i = k+1}^{p_n} i^{-\alpha} \geq \int_{k+1}^{p_n}x^{-\alpha}\,dx,
\end{equation*}
\begin{equation*}
F(p_n) - F(k) \geq \sum_{i = k+1}^{p_n} i^{-\alpha} \geq F(p_n) - F(k+1),
\end{equation*}
where 
\[
	F(x) = \begin{cases} \frac{1}{1-\alpha}x^{1-\alpha},& \text{for }\alpha \neq 1,\\
	\ln(x),&  \text{for }\alpha = 1.\end{cases}
\]

As the sequence can only be benign if $k^\ast = o(n)$, we can only consider values of $k$ that do not exceed some constant fraction of $n$, e.g. $n/2$. 
 Since $p_n = \omega(n)$,  
 noting that, for $x > 0$, the
 sign of $\frac{1}{1 - \alpha} x^{1-\alpha}$ flips
 when $\alpha$ crosses $1$,
 we can write, uniformly for all
 $k \in [1, n/2]$,
 \[
	 \sum_{i = k+1}^{p_n} i^{-\alpha} = \begin{cases}
	 \Theta_\alpha\left(p_n^{1-\alpha}\right),& \text{for }\alpha \in (0,1),\\
	 \Theta_\alpha\left(\ln(p_n/k)\right), &\text{for }\alpha = 1,\\
	 \Theta_\alpha\left(k^{1-\alpha}\right),&  \text{for }\alpha > 1.
	 \end{cases}
\] 
 Recall that we consider $\lambda_{i,n} = i^{-\alpha}$ for $i \leq p_n$. 
Using the formula above, we get uniformly
for all $k \in [1, n/2]$
\[
r_k(\Sigma_n) =  \begin{cases}
\Theta_\alpha\left(k^\alpha p_n^{1-\alpha}\right),& \text{for }\alpha \in (0,1),\\
\Theta_\alpha\left(k\ln(p_n/k)\right), &\text{for }\alpha = 1,\\
\Theta_\alpha\left(k\right),&  \text{for }\alpha > 1.
\end{cases}
\]

Recall that $k^\ast = \min\{k: r_k(\Sigma_n) \geq bn\}$.  We compute
\[
k^\ast =  \begin{cases}
\Theta_\alpha\left( p_n^{1-\frac{1}{\alpha}}  n^{\frac{1}{\alpha}}\right),& \text{for }\alpha \in (0,1),\\
\Theta_\alpha\left(\frac{n}{\ln(p_n/n)}\right), &\text{for }\alpha = 1,\\
\Theta_\alpha\left(n\right),&  \text{for }\alpha > 1.
\end{cases}
\]
One can see that for $\alpha > 1$, $k^\ast = \Omega_\alpha(n)$, so the sequence is not benign for  $\alpha > 1$. On the other hand, $k^\ast = o(n)$ for $\alpha \leq 1$.

Next, analogously to the asymptotics for $r_k(\Sigma)$, we have
\[
r_k(\Sigma_n^2) =  \begin{cases}
\Theta_\alpha\left(k^{2\alpha} p_n^{1-{2\alpha}}\right),& \text{for }\alpha \in (0,0.5),\\
\Theta_\alpha\left(k\ln(p_n/k)\right), &\text{for }\alpha = 0.5,\\
\Theta_\alpha\left(k\right),&  \text{for }\alpha \in (0.5, 1].
\end{cases}
\]

Since $R_k = \frac{r_k(\Sigma)^2}{r_k(\Sigma^2)}$, we can write uniformly for all $k \in [1, n/2]$
\[
R_k =  \begin{cases}
\Theta_\alpha\left(p_n\right),& \text{for }\alpha \in (0,0.5),\\
\Theta_\alpha\left(\frac{p_n}{\ln(p_n/k)}\right), &\text{for }\alpha = 0.5,\\
\Theta_\alpha\left(k^{2\alpha - 1}p_n^{2 - 2\alpha}\right),&  \text{for }\alpha \in  (0.5, 1),\\
\Theta_\alpha\left(\ln(p_n/k)^2\right),&  \text{for }\alpha=1.
\end{cases}
\]

Now we plug in $k^\ast$ instead of $k$. Recall that $p_n/k^\ast = \Theta_\alpha\left((p_n/n)^{1/\alpha}\right)$ for $\alpha \in (0,1)$, and $p_n/k^\ast = \Theta_\alpha\left(p_n/n \ln(p_n/n)\right)$ for $\alpha = 1$. We get

\[
R_{k^\ast} =  \begin{cases}
\Theta_\alpha\left(p_n\right),& \text{for }\alpha \in (0,0.5),\\
\Theta_\alpha\left(n\frac{p_n/n}{\ln(p_n/n)}\right) , &\text{for }\alpha = 0.5,\\
\Theta_\alpha\left(n\left(\frac{p_n}{n}\right)^{\frac{1}{\alpha} - 1}\right),&  \text{for }\alpha \in  (0.5, 1),\\
\Theta_\alpha\left(\ln(p_n/n)^2\right),&  \text{for }\alpha=1.
\end{cases}
\]

Since $p_n = \omega(n)$, for any $\alpha \in (0,1)$, $R_{k^\ast} = \omega(n)$. For $\alpha = 1$ the necessary and sufficient for $R_{k^\ast} = \omega(n)$ is $\ln(p_n/n) = \omega(\sqrt{n})$. 

So far, we obtained the necessary and sufficient conditions for the
last terms to go to zero.
Now let's look at the upper bound for the
first term: since $\lambda_{1,n} \equiv 1$, we just need $r_0/n \to 0$. We
write, for $\alpha \in (0,1]$,
\begin{equation*}
r_0 = \sum_{i = 1}^{p_n} i^{-\alpha}
=\begin{cases}
\Theta_\alpha\left(p_n^{1-\alpha}\right),& \text{for }\alpha \in (0,1),\\
\Theta_\alpha\left(\ln p_n\right),& \text{for }\alpha =1.
\end{cases}\\
\end{equation*}

Thus, for $\alpha < 1$, $r_0(\Sigma_n)/n$ goes to zero if and only if
$p_n = o\left(n^{1/(1-\alpha)}\right)$, and for $\alpha = 1$, $r_0(\Sigma_n)/n$ goes to zero if and only if $\ln(p_n) = o(n)$.

\noindent
{\bf Part~\ref{eigenexample_expanded:expplusconst}:}
Suppose that $\Sigma_n$ is benign. Then because
$R_{k}(\Sigma_n)\le p_n-k$, we must have $p_n=\omega(n)$. Also,
  \begin{align*}
    \tr(\Sigma_n)
      & = \Theta\left(1-e^{-p_n/\tau} + p_n\epsilon_n\right) \\
      & = \Theta\left(1 + p_n\epsilon_n\right),
  \end{align*}
and so $p_n\epsilon_n=o(n)$.
Since $\Sigma_n$ benign implies $k^*=o(n)$, and hence $k^*=o(p_n)$, we
consider $k=o(p_n)$. In this regime,
  \begin{align*}
    \sum_{i>k}\lambda_i
      & = \Theta\left(e^{-k/\tau}-e^{-p_n/\tau}
        + (p_n-k)\epsilon_n\right) \\
      & \le \Theta\left(e^{-k/\tau} + p_n\epsilon_n\right).
  \end{align*}
Thus, whenever $k\le p_n$,
  \begin{align*}
    r_k(\Sigma_n)
      & \le \Theta\left(
        \frac{e^{-k/\tau}+p_n\epsilon_n}{e^{-k/\tau}+\epsilon_n}\right).
  \end{align*}
Notice that
  \begin{align*}
    \frac{d}{dx}\frac{x+p_n\epsilon_n} {x+\epsilon_n}
      = \frac{\epsilon_n-p_n\epsilon_n}{(x+\epsilon_n)^2}
      <0,
  \end{align*}
so $k^*$ must be large enough to make
  \[
    \frac{e^{-k/\tau}+p_n\epsilon_n}{e^{-k/\tau}+\epsilon_n}=\Omega(n).
  \]
Substituting $k=\tau\ln(n/(p_n\epsilon_n))-a$ gives
  \begin{align*}
    r_k(\Sigma_n)
      & \le \Theta\left(
        \frac{p_n\epsilon_n/n + p_n\epsilon_n}
        {p_n\epsilon_n/n + \epsilon_n}\right) \\
      & = \Theta\left(
        \frac{p_n\epsilon_n}
        {p_n\epsilon_n/n}\right) \\
      & = \Theta\left(n\right),
  \end{align*}
which shows that
$k^*\ge\tau\ln(n/(p_n\epsilon_n))-O(1)$.
Thus, if $\Sigma_n$ is benign, we must have $k^*=o(n)$, that is,
$\epsilon_np_n=ne^{-o(n)}$.

Conversely, assume $p_n=\Omega(n)$ and
$\epsilon_np_n=ne^{-o(n)}$ (that is, $\ln(n/(p_n\epsilon_n))=o(n)$).
Set $k=\tau\ln(n/(p_n\epsilon_n))-a$,
for some $a$, which we shall see is $\Theta(1)$. Notice that
$k=o(n)$, so $p_n-k=\Omega(p_n)$ and $e^{-p_n}=o(e^{-k})$. Thus,
  \begin{align*}
    \sum_{i>k}\lambda_i
      & = \Theta\left(e^{-k/\tau}-e^{-p_n/\tau}
        + (p_n-k)\epsilon_n\right) \\
      & = \Theta\left(e^{-k/\tau} + p_n\epsilon_n\right), \\
    \sum_{i>k}\lambda_i^2
      & = \Theta\left(e^{-2k/\tau}-e^{-2p_n}
        + (p_n-k)\epsilon_n^2\right) \\
      & = \Theta\left(e^{-2k/\tau} + p_n\epsilon_n^2\right).
  \end{align*}
These imply
  \begin{align*}
    \tr(\Sigma_n)
      & = \Theta(1+p_n\epsilon_n), \\
    r_k(\Sigma_n)
      & = \Theta\left(
        \frac{e^{-k/\tau}+p_n\epsilon_n}{e^{-k/\tau}+\epsilon_n}\right) \\
      & = \Theta\left(
        \frac{ap_n\epsilon_n/n + p_n\epsilon_n}
        {ap_n\epsilon_n/n + \epsilon_n}\right) \\
      & = \Theta\left(
        \frac{p_n\epsilon_n}
        {ap_n\epsilon_n/n}\right) \\
      & = \Theta\left(n/a\right),
  \end{align*}
which shows that $k^*=\tau\ln(n/(p_n\epsilon_n))+O(1)$.
Also, we have
  \begin{align*}
    R_k(\Sigma_n)
      & = \Theta\left(
        \frac{\left(e^{-k/\tau}+p_n\epsilon_n\right)^2}
        {e^{-2k/\tau}+p_n\epsilon_n^2}\right) \\
      & = \Theta\left(
        \frac{\left(p_n\epsilon_n/n+p_n\epsilon_n\right)^2}
        {p_n^2\epsilon_n^2/n^2+p_n\epsilon_n^2}\right) \\
      & = \Theta\left(\frac{p_n^2\epsilon_n^2}
        {p_n^2\epsilon_n^2/n^2+p_n\epsilon_n^2}\right) \\
      & = \Theta\left(\min\left\{n^2,p_n \right\}\right).
  \end{align*}
Combining gives
  \[
    R(\hat\theta)
      = O\left(\frac{\epsilon_np_n+1}{n}
      + \frac{\ln(n/(\epsilon_np_n))}{n}
      + \max\left\{\frac{1}{n},\frac{n}{p_n}\right\}\right).
  \]
Now, it is clear that $p_n=\omega(n)$, $\epsilon_np_n=o(n)$, and
$\epsilon_np_n=ne^{-o(n)}$ imply that $\Sigma_n$ is benign.

\section{Upper bound on the $B$ term}\label{a:upperB}

We can control the term ${\theta^\ast}^\top B \theta^\ast$ in
Lemma~\ref{lemma:bv} using a standard argument.

\begin{lemma}\label{lemma:bias}
There is a constant $c$, that depends only on $\sigma_x$, such that for any $1<t<n$, with
probability at least $1-e^{-t}$,
  \[
    {\theta^\ast}^\top B \theta^\ast
        \leq c\|\theta^\ast\|^2\|\Sigma\|\max
        \left\{\sqrt{\frac{r_0(\Sigma)}{n}},
        \frac{r_0(\Sigma)}{n}, \sqrt{\frac{t}{n}} \right\}.
  \]
\end{lemma}

\begin{proof}
Note that
  \begin{align}\label{equation:Xorthog}
    \left(I - X^\top \left(X X^\top \right)^{-1}X\right)X^\top
    &= X^\top  - X^\top \left(X X^\top \right)^{-1}(XX^\top ) =0.
  \end{align}
Moreover, for any $v$ in the orthogonal complement to the span of
the columns of $X^\top$,
  \[
    \left(I - X^\top \left(X X^\top \right)^{-1}X\right)v = v.
  \]
Thus, 
  \begin{equation}\label{eqn:projection}
    \|I - X^\top \left(X X^\top \right)^{-1}X\|\leq 1.
  \end{equation}
Now we can apply~\eqref{equation:Xorthog} to write
  \begin{align*}
    {\theta^\ast}^\top B \theta^\ast
    & 
     = {\theta^\ast}^\top \left(I - X^\top
      \left(X X^\top \right)^{-1}X\right)\Sigma\left(I - X^\top
      \left(X X^\top \right)^{-1}X\right)\theta^\ast \\
    &= {\theta^\ast}^\top \left(I
      - X^\top \left(X X^\top \right)^{-1}X\right)\left(\Sigma -
      \frac1n X^\top X\right) 
        \left(I - X^\top
      \left(X X^\top \right)^{-1}X\right)\theta^\ast.
  \end{align*}
Combining with~\eqref{eqn:projection} shows that
  \[
    {\theta^\ast}^\top B \theta^\ast
    \leq \left\|\Sigma - \frac1n X^\top X\right\|\|\theta^\ast\|^2.
  \]
Thus, due to Theorem~9 in~\cite{koltchinskii2017},
there is an absolute constant $c$ such that for any $t > 1$  with
probability at least $1-e^{-t}$,
  \[
    {\theta^\ast}^\top B \theta^\ast
    \leq c\|\theta^\ast\|^2\|\Sigma\|\max\left\{\sqrt{\frac{r(\Sigma)}{n}},
      \frac{r(\Sigma)}{n}, \sqrt{\frac{t}{n}}, \frac{t}{n} \right\},
  \]
where 
  \[
    r(\Sigma) := \frac{(\E\|x\|)^2}{\|\Sigma\|}
    \leq \frac{\tr(\Sigma)}{\|\Sigma\|}
    = \frac{1}{\lambda_1}\sum_i \lambda_i = r_0(\Sigma).
  \]
\end{proof}

\section{Another lower bound}\label{a:packing}

In this section, we prove 
the second paragraph of Theorem~\ref{th::main}.

First, note that, without loss of generality,
$|| \Sigma ||_2 = 1$ and $|| \theta^* || = 1$,
since scaling these scales the excess risk by 
$|| \Sigma ||_2$ and $|| \theta^* ||^2$ respectively.
This implies $\lambda_1 = 1$, and we may further
assume without loss of generality that
$\Sigma = \diag(\lambda_1, \lambda_2,\ldots)$.  Define
$s = \sum_{i=1}^{\infty} \lambda_i$.  
We may also assume that
\begin{equation}
\label{e:r0.big}
\frac{r_0(\Sigma)}{n
        \log\left(1+r_0(\Sigma)\right)}\ge c_2
\end{equation}
since, otherwise, the lower bound is vacuously satisfied.

Define a metric $\rho$ over $\bbH$ by
\[
\rho(u,v) = \sqrt{ (u - v)^{\top} \Sigma (u - v) },
\]
so that, informally, a successful learning algorithm
achieves $\rho(\htheta,\theta) < \sqrt{\tau_0}$.

\begin{definition}
\label{d:S}
Define sets $S_1, S_2,...$ of indices as follows.  
Let $S_1 = \{ 1 \}$;  let $S_2 = \{ 2,..., i_2 \}$, for
the least $i_2$ such that $\sum_{i=2}^{i_2} \lambda_i \geq 1$.
Continue the same way as long as possible; for all $j > 2$, let
$S_j = \{ i_{j-1},...,i_j \}$, where $i_j$ is the least index
such that $\sum_{i=i_{j-1}}^{i_j} \lambda_i \geq 1$.
\end{definition}

\begin{lemma}
Definition~\ref{d:S} produces $\Omega(n \log n)$ sets.
\end{lemma}
\begin{proof}
For all $j$, $\sum_{i \in S_j} \lambda_i < 2$.  
Thus, for all $k$, 
$\sum_{i \leq i_k} \lambda_i = \sum_{j \leq k} \sum_{i \in S_j} \lambda_i < 2 k$.  
Assume for contradiction that, for $k < \frac{c_2 n \ln n}{4}$, 
after $S_k$,
it is not possible to add any more sets.  Then
$\sum_{i \leq i_k} \lambda_i < \frac{c_2 n \ln n}{2}$,
and, since no more sets can be added,
$\sum_{i = 1}^{\infty} \lambda_i < 1 + \frac{c_2 n \ln n}{2}$.
We claim that, for large enough $n$, this contradicts the assumption that
$\frac{\sum_{i = 1}^{\infty} \lambda_i}
      {\ln \left(1 + \sum_{i = 1}^{\infty} \lambda_i \right)} \geq c_2 n$.
To see why, consider the function $\phi : \R^+ \rightarrow \R^+$ defined
by $\phi(s) = \frac{s}{\ln (1 + s)}$.  The function $\phi$ is
increasing for $s \geq 1$, so it suffices to show that
$\phi\left( 1 + \frac{c_2 n \ln n}{2} \right) < c_2 n$,
and 
\begin{align*}
\phi\left( 1 + \frac{c_2 n \ln n}{2} \right) 
  &  = \frac{1 + \frac{c_2 n \ln n}{2}}{\ln \left(2 + \frac{c_2 n \ln n}{2} \right)} \\
  &  = \left( \frac{c_2}{2} + o(1) \right) n,
\end{align*}
yielding the contradiction and completing the proof.
\end{proof}

\begin{definition}
If the number of sets produced by the process
of Definition~\ref{d:S} is finite, let $d$ be this
finite number.  Otherwise, let $d = \lceil n \ln n \rceil$.
\end{definition}

Now, informally, we, in our role as an adversary, commit to
assigning all covariates in $S_j$ the same weight.  The
following definition formalizes this idea.
\begin{definition}
Define a mapping $\phi$ from $\R^d$ to $\bbH$ as follows.
For $w \in \R^d$, $\phi(w) = \theta$ where, for all
$j$ in $[d]$, for all $i$ in $S_j$, $\theta_i = w_j$.
For all $i > i_d$, $\theta_i = 0$.
\end{definition}

We would like to show that applying $\phi$ to an
$L_2$ packing yields a $\rho$-packing, which is done
in the following lemma.
\begin{lemma}
\label{l:rho.by.ell2}
For all $u,v \in \R^n$, $\rho(\phi(u), \phi(v)) \geq || u - v ||$.
\end{lemma}
\begin{proof}
\begin{align*}
\rho(\phi(u), \phi(v))^2 
& = \sum_i \lambda_i (\phi(u)_i - \phi(v)_i)^2 \\
& = \sum_j \left( \sum_{i \in S_j} \lambda_i \right) (u_j - v_j)^2 \\
& \geq \sum_j (u_j - v_j)^2.
\end{align*}
\end{proof}

Let $A$ be the least-norm interpolation algorithm.  
We will bound the accuracy of $A$ by bounding its
performance in terms of an algorithm $C$ built using
$A$ as a subroutine, as was done in a related context
in \cite{bartlett1996fat}.  The definition of Algorithm
$C$ is illustrated in Figure~\ref{fig:C}, which is reproduced
from \cite{bartlett1996fat}.  
\begin{figure}[!ht]
    \centering
   \includegraphics[width=0.4\textwidth]{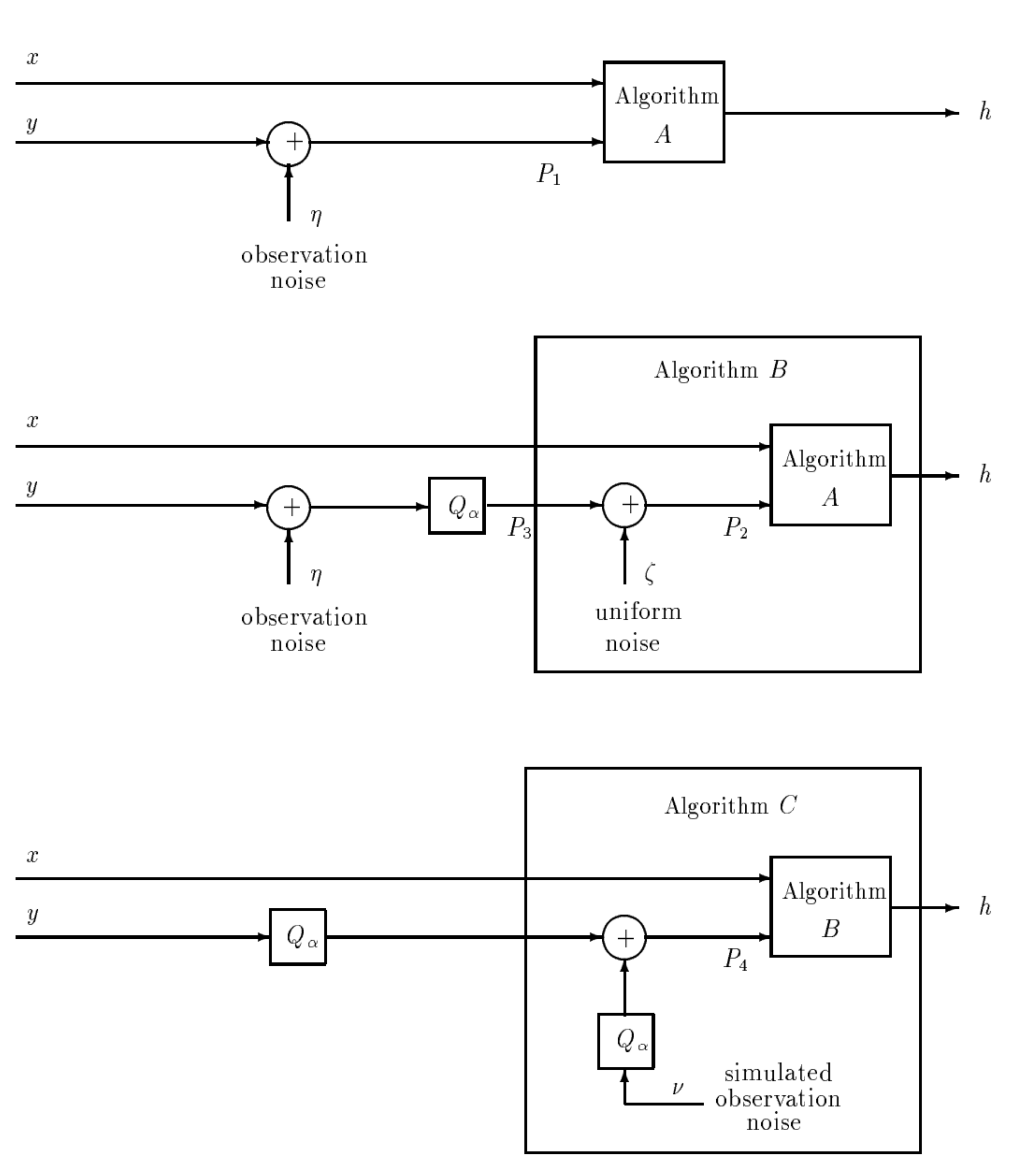}
    \caption{A diagram illustrating the definition of Algorithm
             $C$.}
\label{fig:C}
\end{figure}
The definition uses the
function $Q_{\alpha}$ that rounds its input to the nearest
multiple of $\alpha$.  Algorithm $C$ applies algorithm
$A$ to training data whose response variables have been
modified.  For each example $(x,y)$, and
simulated artificial noise $\eps$ distributed as
$N(0,1)$, and artificial noise $\zeta$ distributed uniformly
on $(-\alpha/2,\alpha/2)$, 
Algorithm $C$ gives $(x, y + Q_{\alpha}(\eps) + \zeta)$ to
$A$.  
The following lemma is similar to Lemma 5
of \cite{bartlett1996fat}.  One important difference is that
we show that Algorithm $C$ approximates the linear
function parameterized by $\theta^*$, not its discretization.
\begin{lemma}
\label{l:to.quantized}
If the linear interpolant algorithm $A$ 
has error $\tau$ from $n$ examples 
drawn from $N(0,\Sigma)$ with independent $N(0,1)$ noise with probability
$1 - \delta$, and
\[
\alpha \leq \min \left\{ \frac{\delta}{2 n}, 2 \tau \right\}
\]
then, in the absence of noise, Algorithm $C$, given
$n$ examples of the form $(x,Q_{\alpha}(\theta^{\top} x))$, 
with probability $1 - 2 \delta$,
achieves $\rho(\htheta, \theta^*)^2 \leq \tau$.
\end{lemma}

The proof of Lemma~\ref{l:to.quantized} will be
deferred until we have proved some more lemmas.

Recall the definition of total
variation distance, $d_{TV}(P,Q) = \sup_E | P(E) - Q(E) |$.
The following lemma is implicit in the proof
of Lemma 6 of \cite{bartlett1996fat}.
\begin{lemma}
\label{l:quant.approx}
Let $\eta, \nu$ be random variables that are distributed
according to $N(0,1)$ and let $\zeta$ be uniform over
$[-\alpha/2,\alpha/2]$.  
\begin{enumerate}
\item[(a)] For any $y \in \R$, if
$P_1$ is the distribution of $y+\eta$
and $P_2$ is the distribution of 
$Q_{\alpha} (y + \eta) + \zeta$, we have 
$d_{TV}(P_1, P_2) \leq \alpha$.
\item[(b)] For any $y \in \R$, if
$P_3$ is the distribution of $Q_{\alpha} (y + \eta)$
and $P_4$ is the distribution of 
$Q_{\alpha} (y) + Q_{\alpha}(\eta)$, 
$d_{TV} (P_3, P_4) \leq \alpha$.
\end{enumerate}
\end{lemma}

We will use the following, which is implicit
in the proof of Lemma 8 of \cite{bartlett1992learning}.
\begin{lemma}
\label{l:product}
If $P_1,...,P_n, Q_1,...,Q_n$ are probability distributions over
a domain $U$, and $\chi$ is a $[0,1]$-valued random variable
defined on $U^n$ then
\[
\left| \E_{\prod_t P_t} (\chi) - \E_{\prod_t Q_t} (\chi) \right|
 \leq \sum_{t=1}^n d_{TV}(P_t,Q_t).
\]
\end{lemma}

\medskip

Now, we are ready to prove Lemma~\ref{l:to.quantized}.  The
proof closely follows the proof of Lemma 5 in \cite{bartlett1996fat}.

{\em Proof (of Lemma~\ref{l:to.quantized})}.
Let $A(X, \bs{\eps}, \theta^*)$ be the output $\htheta$ of
the least-norm interpolant when the covariates are $X$,
the noise is $\bs{\eps}$, and the target is $\theta^*$.  
Let $A(X, \bs{y})$ be the output $\htheta$ of
the least-norm interpolant when the covariates are $X$,
and the response variables are $\bs{y}$, and 
let $\sam(X,\bs{\eps},\theta^*) = (X, X \theta^* + \bs{\eps})$ be
the input arising from covariates $X$, regressor $\theta^*$ and
noise $\bs{\eps}$.  

By assumption
\[
N(0,\Sigma)^n \times N(0,1)^n
 \{ (X, \bs{\eps}) 
       :
     \rho( A(\sam(X, \bs{\eps}, \theta^*)), \theta^*)^2 \geq \tau \} < \delta.
\]
Let $\zeta_t$ be a random variable with distribution $U_{\alpha}$,
where $U_{\alpha}$ is the uniform distribution over $(-\alpha/2, \alpha/2)$.
Let $B$ be the randomized algorithm that adds noise 
$\zeta_t$ to each $y_t$ value it receives, passes the result to
Algorithm $A$, and returns $A$'s output.

Fix $X$, and define
\begin{align*}
E & = \{ \bs{\eps} \in \R^n: 
      \rho( A(\sam(X, \bs{\eps}, \theta^*)), \theta^*)^2 \ge \tau \} \\
E_1 & = \{ \bs{y} \in \R^n: 
      \rho( A(X,\bs{y}), \theta^*)^2 \ge \tau \}.
\end{align*}
We have
\[
N(0,1)^n (E) = \left( \prod_{t=1}^n P_{1|x_t} \right) (E_1),
\]
where $P_{1|x_t}$ is the distribution of $(\theta^*)^{\top} x_t + \eps_t$.

Define $P_{2|x_t}$ as the distribution of
$Q_{\alpha} ((\theta^*)^{\top} x_t + \eps_t) + \zeta_t$.
From Lemma~\ref{l:quant.approx},
$d_{TV}(P_{1|x_t}, P_{2|x_t}) \leq \alpha$.
Applying Lemma~\ref{l:product} with $\chi$ as
the indicator function for $E_1$,
\[
\left| 
\left( \prod_t P_{2|x_t} \right) (E_1) - \left( \prod_t P_{1|x_t} \right) (E_1) 
\right|
 \leq \alpha n.
\]
Since $\alpha \leq \frac{\delta}{2 n}$, this implies
\begin{align*}
\left( \prod_t P_{2|x_t} \right) (E_1)
& \leq \left( \prod_t P_{1|x_t} \right) (E_1)
      + \delta/2 
= N(0,1)^n (E) + \delta/2.
\end{align*}
Let $P_{3|x_t}$ be the distribution of $Q_{\alpha} ((\theta^*)^{\top} x_t + \eps_t)$,
and let
\[
E_3 =  \{ (\bs{y}, \bs{\zeta}) \in \R^n \times \R^n:
      \rho( A(X,\bs{y} + \bs{\zeta}), \theta^*)^2 > \tau
       \}
\]
so that
\[
\left( \prod_t P_{2|x_t} \right) (E_1)
 = \left( \prod_t ( P_{3|x_t} \times U_{\alpha}^n) \right) (E_3).
\]
Let $P_{4|x_t}$ be the distribution of 
$Q_{\alpha} ((\theta^*)^{\top} x_t) + Q_{\alpha} (\eps_t)$.
Applying Lemma~\ref{l:product}, we get
\begin{align*}
& \left|
\left( \prod_t ( P_{3|x_t} \times U_{\alpha}^n) \right) (E_3)
 - \left( \prod_t ( P_{4|x_t} \times U_{\alpha}^n) \right) (E_3)
\right| 
\leq \sum_{t=1}^m d_{TV}(P_{3|x_t}, P_{4|x_t}).
\end{align*}
From Lemma~\ref{l:quant.approx}, $d_{TV}(P_{3|x_t}, P_{4|x_t}) \leq \alpha$,
so
\begin{align*}
\left( \prod_t ( P_{4|x_t} \times U_{\alpha}^n) \right) (E_3)
 & \leq \left( \prod_t ( P_{3|x_t} \times U_{\alpha}^n) \right) (E_3)
      + \delta/2 \\
 & = \left( \prod_t P_{2|x_t} \right) (E_1)
      + \delta/2 \\
 & \leq N(0,1)^n (E)
      + \delta.
\end{align*}
Averaging over the random choice of $X$,
the probability, for 
$(X, \bs{\zeta}, \bs{\eps})$
distributed as $N(0,\Sigma)^n \times U_{\alpha}^n \times N(0,1)^n$,
that
$\rho(A (X, Q_{\alpha} (X \theta^*)
                    + Q_{\alpha} (\bs{\eps})
                    + \bs{\zeta})),
         \theta^*)^2 > \tau$, is at most
\begin{align*}
& (N(0,\Sigma)^n \times N(0,1)^n)
 \{ (X, \bs{\eps}) :
    \rho(A (\sam(X, \bs{\eps}, \theta^*), \theta^*)^2 > \tau \}  + \delta 
  \;\leq\; 2 \delta.
\end{align*}
But $A (X, Q_{\alpha} (X \theta^*)
                    + Q_{\alpha} (\bs{\eps})
                    + \bs{\zeta})$ is the output of
the randomized algorithm $C$, so this completes the proof. \qed

So, informally, we have shown that if the least norm interpolant
can learn unit length weight vectors with noise and $N(0,\Sigma)$
data, then there is an algorithm $C$ than can learn
from quantized data without noise.  The next step is to lower
bound the error of $C$.

Recall that we have fixed an $n$, that
$s \eqdef \sum_{i=1}^{\infty} \lambda_i \geq c n$, and that
$\Sigma = \diag(\lambda_1, \lambda_2,...)$.  

We will use the following, which is
an immediate consequence
of Corollary~\ref{cor:Bernstein}.
\begin{lemma}
\label{l:x.short}
For each row $x_t$ of $X$, and each $q > 1$,
\[
\Pr(|| x_t || > q \sqrt{s}) \leq \exp(-q^2/c).
\]
\end{lemma}

The proof of the following lemma borrows heavily
from \cite{benedek1991learnability}.
\begin{lemma}
\label{l:bi}
If $1/\alpha = O(n)$, there is a constant $\tau$ such that,
for any regression algorithm $C$, 
for all large enough $n$, if $C$ is given
$n$ examples of the form $(X, Q_{\alpha} (X \theta^*))$,
if the rows of $X$ are $n$ independent draws from $N(0,\Sigma)$, 
with probability at least $1/2$, its output 
$\htheta$ satisfies $\rho(\htheta, \theta^*)^2 > \tau$.
\end{lemma}
\begin{proof}
For $\tau > 0$ to be chosen later,
assume for contradiction that, with probability
$1/2$, $\rho(\htheta, \theta^*)^2 \leq \tau$.
For an absolute constant $c_3$, let $G$ be a set
of $(1/\tau)^{ c_3 d}$ members of the unit ball in $\bbH$ that
are pairwise separated by $3 \sqrt{\tau}$ w.r.t.\ $\rho$
so that, for distinct members $g,h$ of $G$,
$\rho(g,h)^2 > 9 \tau$.  

For each $X \in \R^{n \times \infty}$, and each 
$\theta \in \bbH$, define
\[
\phi(X, \theta)
 = 
 \left\{ \begin{array}{ll}
    1 & \mbox{if $\rho(C(X, Q_{\alpha} (X \theta)),\theta)^2 \leq \tau$} \\
    0 & \mbox{otherwise}
          \end{array}
     \right.
\]
and define
\[
S = \E_X \left[
      \sum_{\theta \in G}
      \phi(X, \theta)
    \right].
\]
Our assumption about the learning ability of $C$ implies that
\begin{equation}
\label{e:S.big}
S \geq |G|/2 = (1/\tau)^{c_3 d}/2.  
\end{equation}
For any $g,h \in G$
for which $Q_{\alpha} (X g) = Q_{\alpha} (X h)$,
since $\rho(g,h) > 3 \sqrt{\tau}$, it cannot be the
case that both $\phi(X, g)$ and $\phi(X,h)$ are both
$1$.  Thus, recalling that $x_1,...,x_n$ are the
rows of $X$, and that all elements of $G$ have length
at most 1, we have
\begin{align*}
S  & \leq \E_X (| \{ Q_{\alpha} (X g): g \in G \}|) \\
 & = \E_X (| \{ Q_{\alpha} (X g): g \in G \}| 
               \ind_{\max_t || x_t || < \sqrt{s}})  
           + \sum_{i = 1}^{\infty}
          \E_X (| \{ Q_{\alpha} (X g): g \in G \}| 
                 \ind_{\floor{\max_t || x_t ||/s} = i}) \\
 & \leq \left( \frac{c_4 \sqrt{s}}{\alpha} \right)^n
           + \sum_{i = 1}^{\infty}
         (i\sqrt{s}/\alpha)^n \times \Pr(\max_t || x_t || \geq i \sqrt{s}) \\
 & \leq \left( \frac{c_4 \sqrt{s}}{\alpha} \right)^n
           + \sum_{i = 1}^{\infty}
         (i\sqrt{s}/\alpha)^n \times 
       n e^{-i^2/c_5}
     \;\;\;\mbox{(by Lemma~\ref{l:x.short})} \\
 & \leq c_6 n \left( \frac{c_4 \sqrt{s}}{\alpha} \right)^n.
\end{align*}
Since $1/\alpha = O(n)$
\[
|\{ Q_{\alpha} (X g): g \in G \}|
 \leq \exp( O(n \log (n s)) )
 = \exp( O(n \log (n d)) )
\]
since $d = \Theta(s)$.  Since $d = \Omega(n \log n)$,
for large enough $n$ and small enough $\tau$, this
contradicts \eqref{e:S.big}, completing the proof.
\end{proof}

Now we are ready to put everything together to prove
the second paragraph of Theorem~\ref{th::main}.
By Lemma~\ref{l:to.quantized}, it suffices to
prove that, for a small enough constant $\tau_0$,
if $1/\alpha = O(n)$, with probability
$1/2$, Algorithm $C$, given examples 
$(x,Q_{\alpha}(\theta^{\top} x))$, 
with probability $1/2$, fails to
achieves $\rho(\htheta, \theta^*)^2 \leq \tau_0$.
By Lemma~\ref{l:bi}, this is the case, completing the
proof.

\end{document}